%% file: icml2026.tex
\documentclass{article}
\usepackage{microtype}
\usepackage{graphicx}
\usepackage{subcaption}
\usepackage{booktabs} 
\usepackage{hyperref}
\makeatletter
\g@addto@macro\UrlBreaks{\do\-\do\_}
\makeatother
\usepackage[preprint]{icml2026}
\usepackage{amsmath}
\usepackage{amssymb}
\usepackage{mathtools}
\usepackage{amsthm}
\usepackage[capitalize,noabbrev]{cleveref}
\theoremstyle{plain}
\newtheorem{theorem}{Theorem}[section]
\newtheorem{proposition}[theorem]{Proposition}

\theoremstyle{definition}

\theoremstyle{remark}

\usepackage[textsize=tiny]{todonotes}

\usepackage[table]{xcolor}
\usepackage{tabularx}
\usepackage{array}
\definecolor{rowgray}{gray}{0.93}
\newcolumntype{Y}{>{\centering\arraybackslash}X}

\usepackage{caption}
\usepackage{listings}
\usepackage[most]{tcolorbox}
\tcbuselibrary{listings,breakable,skins}

\lstdefinestyle{icmlpy}{
  language=Python,
  basicstyle=\ttfamily\small,
  columns=fullflexible,
  keepspaces=true,
  showstringspaces=false,
  breaklines=true,
  breakatwhitespace=false,
  tabsize=4
}

\newtcblisting{icmlcodebox}{
  listing only,
  listing options={style=icmlpy},
  colback=black!3,
  colframe=black!55,
  boxrule=0.6pt,
  arc=3mm,
  left=6pt,
  right=6pt,
  top=6pt,
  bottom=6pt,
  boxsep=0pt,
  breakable
}

\icmltitlerunning{Beyond Variance: Prompt-Efficient RLVR via Rare-Event Amplification and Bidirectional Pairing}

\begin{document}

\twocolumn[
   \icmltitle{Beyond Variance: Prompt-Efficient RLVR via Rare-Event Amplification and Bidirectional Pairing}

  \begin{icmlauthorlist}
    \icmlauthor{Yujuan Pang}{yyy}
    \icmlauthor{Jiaxin Li}{comp}
    \icmlauthor{Xin Sheng}{yyy}
    \icmlauthor{Ran Peng}{comp}
    \icmlauthor{Yong Ma}{sch,qax}
  \end{icmlauthorlist}

  \icmlaffiliation{yyy}{Beijing University of Post and Telecommunications.}
  \icmlaffiliation{comp}{Sichuan Agricultural University.}
  \icmlaffiliation{sch}{QI-ANXIN Group.}
  \icmlaffiliation{qax}{RapidAI Research}

  \icmlcorrespondingauthor{Yong Ma}{}

  \icmlkeywords{Machine Learning, ICML}

  \vskip 0.3in
]



\begin{NoHyper}
\printAffiliationsAndNotice{}  
\end{NoHyper}

\input{sec/01_abstract}
\input{sec/02_introduction}
\input{sec/03_related_work}

\input{sec/04_method}

\input{sec/05_experimental_setup}
\input{sec/06_conclusion}

\input{sec/08_impact_statement}

\bibliography{icml2026}
\bibliographystyle{icml2026}

\input{sec/07_appendix}

\end{document}

%% file: sec/01_abstract.tex
\begin{abstract}
Reinforcement learning with verifiable rewards (RLVR) is effective for training large language models on deterministic outcome reasoning tasks. Prior work shows RLVR works with few prompts, but prompt selection is often based only on training-accuracy variance, leading to unstable optimization directions and weaker transfer. We revisit prompt selection from a mechanism-level view and argue that an effective minibatch should provide both (i) a reliable positive anchor and (ii) explicit negative learning signals from rare failures. Based on this principle, we propose \emph{positive--negative pairing}: at each update, we sample a hard-but-solvable $q^{+}$ and an easy-but-brittle prompt $q^{-}$(high success rate but not perfect), characterized by low and high empirical success rates under multiple rollouts. We further introduce Weighted GRPO, which reweights binary outcomes at the pair level and uses group-normalized advantages to amplify rare successes on $q^{+}$ into sharp positive guidance while turning rare failures on $q^{-}$ into strong negative penalties. This bidirectional signal provides informative learning feedback for both successes and failures, improving sample efficiency without suppressing exploration. On Qwen2.5-Math-7B, a single paired minibatch per update consistently outperforms a GRPO baseline that selects two prompts via commonly used variance-based selection heuristics: AIME~2025 Pass@8 improves from 16.8 to 22.2, and AMC23 Pass@64 from 94.0 to 97.0, while remaining competitive with large-scale RLVR trained from a pool of 1209 training prompts. Similar gains are observed on Qwen2.5-Math-7B-Instruct.
\end{abstract}

%% file: sec/02_introduction.tex
\section{Introduction}
Large language models (LLMs) have recently achieved substantial progress in mathematical reasoning, demonstrating strong performance on challenging benchmark problems \citep{guo2025deepseekr1,team2025kimik15}. A key driver behind this progress is reinforcement learning with verifiable rewards (RLVR), which trains models on deterministic-outcome tasks using rewards computed by automatic verifiers rather than subjective human judgments \citep{jaech2024o1, lambert2024tulu3, gao2024rewardreasoning,team2025kimik15}. In mathematical reasoning, verifier feedback is typically outcome-based and unambiguous, and is often binary, e.g., $r\in\{0,1\}$ (incorrect/correct) or equivalently $r\in\{-1,1\}$ depending on the implementation. This setup reduces reward hacking \citep{miao2024inform, cai2025verifiablenoisy} and avoids large-scale human labeling or training a separate reward model \citep{li2025processqvalue, zhang2025processlessons}.

Despite these advantages, training-prompt selection for RLVR remains poorly understood, especially in the very-low-data regime \citep{li2024numinamath,luo2025deepscaler, yu2025dapo}. While prior work has curated high-quality datasets and prompt collections for mathematical reasoning, it is still unclear which prompts should be used for RLVR when only a handful of prompts are feasible. Fundamental questions remain open: how much data is actually needed, which prompts matter most, and how prompt quality and quantity shape RLVR outcomes. A closely related effort, LIMR \citep{li2025limr}, proposes learning impact measurement to score prompts and shows that performance can be largely maintained while shrinking the RLVR prompt set by about sixfold, but does not characterize how far such compression can be pushed before performance breaks down. More recently, RLVR with a single training prompt shows that substantial gains can arise from extremely few prompts \citep{wang2025onetrainingexample}. In that setting, prompts are ranked by a simple heuristic: the historical variance of training accuracy, yet the authors emphasize that this criterion is not necessarily optimal and that many moderate-/low-variance prompts can perform comparably well. In parallel, a growing body of work highlights that explicitly penalizing incorrect trajectories can be surprisingly effective: it suppresses wrong generations and redistributes probability mass toward other plausible solutions under the model prior \citep{zhu2025negativeRLVR,yang2025negativesampleaugmentation,chen2025sgpo,feng2025dontwastemistakes,arnal2025asymmetricreinforce}.

Motivated by these observations, one open problem is that: \emph{How can we select training prompts for RLVR to maximally improve a base model's mathematical reasoning while using as few prompts as possible?}

\paragraph{A mechanism-level view: bidirectional teaching signals from tail events.}
Our starting point is a simple but consequential observation about RLVR under sparse binary rewards: policy-gradient updates can be dominated by \emph{rare tail events}: occasional successes on hard prompts or occasional failures on easy prompts. In the very-low-data regime, these rare events may or may not occur in a given minibatch, making the update direction sensitive to the particular samples drawn, leading to unstable update directions. This suggests that effective prompt selection should ensure that each update contains (i) a hard-but-solvable positive anchor that provides a clear “do” signal, and (ii) an easy-but-brittle prompt whose rare failures provide an explicit “don’t” signal, rather than selecting prompts solely by hardness or training-accuracy variance. Concretely, rare successes on a hard-but-solvable prompt provide sharp positive teaching signals, while rare failures on an easy-but-brittle prompt provide sharp ``do-not'' signals. Pairing these two regimes concentrates learning on informative tail events.

\paragraph{Main finding: two prompts can be sufficient.}
Guided by this mechanism, we find that in the extreme low-data setting, RLVR can be highly effective with only two carefully chosen prompts: one hard-but-solvable prompt and one easy-but-brittle prompt. Here, hard-but-solvable means the prompt has a low-but-nonzero success rate under a fixed number of rollouts, so that correct solutions occur as rare events; easy-but-brittle means the prompt has a high (but not perfect) empirical success rate under the current policy, and thus produces occasional failures that provide explicit negative learning signals. Empirically, this two-prompt design consistently outperforms two-prompt baselines chosen by commonly used heuristics (e.g., variance-based selection), and recovers a substantial fraction of the gains obtained by training on much larger prompt pools (e.g., 1209 prompts) across multiple mathematical reasoning benchmarks.

\paragraph{Approach: positive--negative pairing with Weighted GRPO.}
To reliably instantiate these bidirectional signals with only two prompts, we propose \emph{positive--negative pairing}. At each update, we select (i) a \emph{positive anchor} $q^{+}$ in a low-but-nonzero success regime $p(q^{+})\in[1/G,c/G]$ so that rare successes are amplified into strong positive advantages, and (ii) a \emph{negative guidance} $q^{-}$ in a high-but-not-perfect success regime $p(q^{-})\in[1-c/G,1-1/G]$ so that rare failures are amplified into strong negative advantages (we use $G{=}8$). We further introduce \emph{Weighted GRPO} (WGRPO), which applies group-normalized advantages to weighted binary outcomes and thereby implements \emph{rare-event amplification}: it upweights rare successes when $p$ is small and upweights rare failures when $p$ is large, while avoiding degenerate all-correct/all-wrong groups with collapsed within-group variance. In practice, we instantiate pairing via a lightweight probing stage on a structured candidate pool: we draw the negative-guidance pool from DeepScaleR-sub (typically easier) and the positive-anchor pool from AIME~2025 (typically harder), estimate success rates under the current policy, discard near-0/near-1 candidates, and select $q^{+},q^{-}$ by targeting $p_{\text{hard}}\approx 1/G$ and $p_{\text{easy}}\approx 1-1/G$.

Our main contributions are:
\begin{itemize}
  \item \textbf{Bidirectional prompt selection via positive--negative pairing.}
  We provide a mechanism-level view of prompt selection in low-data RLVR and propose a minimal two-prompt design: one easy-but-brittle prompt that yields rare failures (a strong ``do-not" signal) and one hard-but-solvable prompt that yields rare successes (a strong ``do" signal). This pairing provides complementary teaching signals that make each update more informative and less dominated by one-sided outcomes, improving prompt efficiency.

  \item \textbf{Weighted GRPO for rare-event amplification under binary rewards.}
  We introduce WGRPO, which reweights binary outcomes and applies group-normalized advantages to amplify rare successes on $q^{+}$ and rare failures on $q^{-}$, encouraging exploration while stabilizing update directions under sparse outcome rewards.

  \item \textbf{Consistent gains on mathematical reasoning benchmarks with only two prompts.}
  We evaluate on Qwen2.5-Math-7B and observe consistent Pass@$k$ improvements on MATH500, AIME~2025, and AMC23, with representative gains such as 16.8$\rightarrow$22.2 (AIME~2025 Pass@8) and 94.0$\rightarrow$97.0 (AMC23 Pass@64). Moreover, our bidirectional prompt selection recovers a substantial fraction of the gains achieved by large-scale RLVR trained on 1209 prompts, and surpasses it on AMC23 at larger $k$. Similar gains hold for Qwen2.5-Math-7B-Instruct.
\end{itemize}

\begin{figure*}[t]
  \centering
  \includegraphics[width=\textwidth]{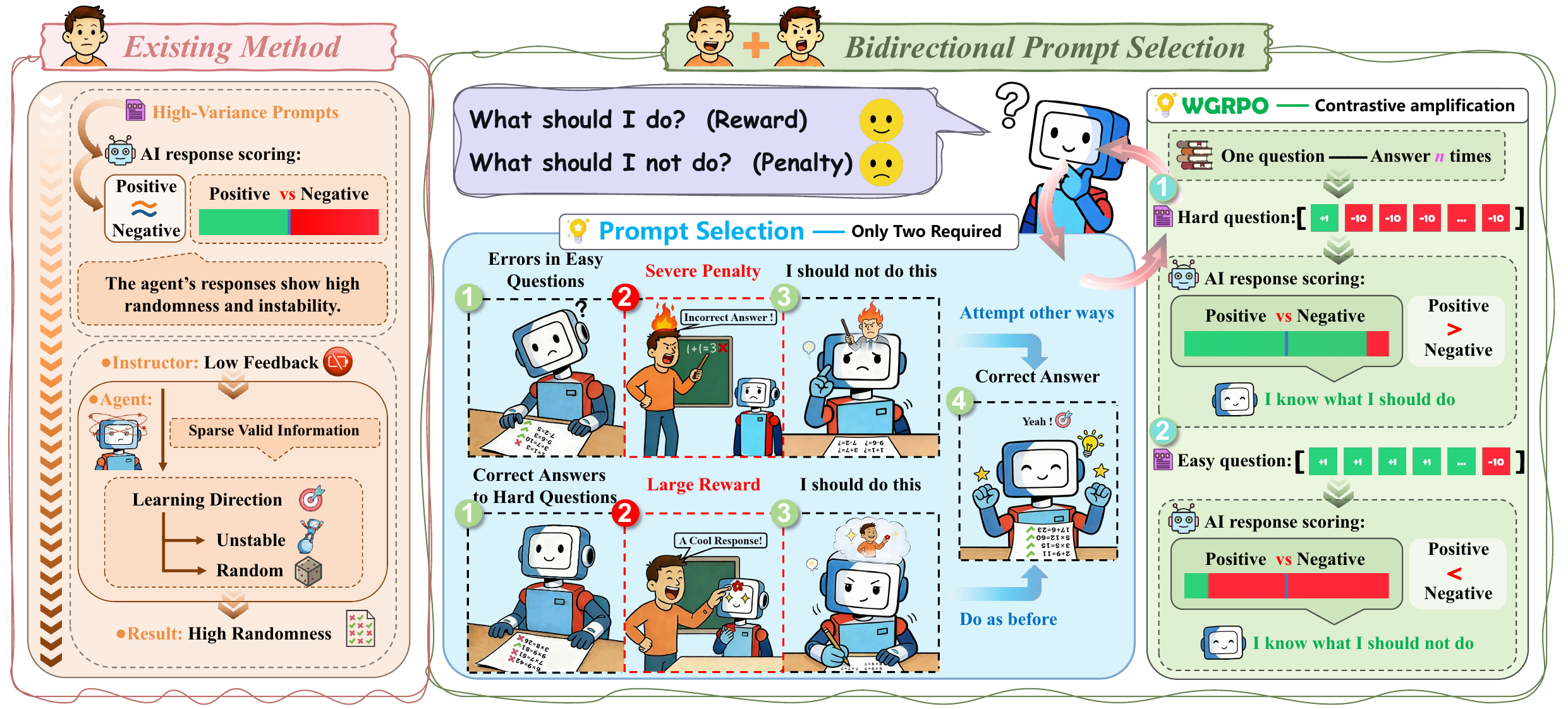}
  \caption{Overview of bidirectional prompt selection and WGRPO. A common low-data RLVR baseline is to prioritize ``high-variance'' prompts, which can be sensitive to sampling noise. We instead select a two-prompt positive--negative pair: a hard prompt where rare successes provide strong positive guidance, and an easy prompt where rare failures provide strong negative penalties. WGRPO contrastively amplifies these tail events across repeated rollouts, encouraging exploration while stabilizing update directions.}
  \label{fig:strategy}
\end{figure*}

%% file: sec/03_related_work.tex
\section{Related Work}

\paragraph{Reinforcement Learning with Verifiable Rewards (RLVR).}
RLVR improves LLM reasoning by using automatic verifiers to provide outcome-based rewards, avoiding human preference modeling and enabling scalable policy-gradient training. This paradigm is particularly effective for mathematical reasoning, where verification is often based on exact answer matching or other deterministic checks and the reward is typically binary~\citep{shao2024deepseekmath,primeintellect2025intellect2,xu2025elasticreasoning,wei2025truthrl,wang2025onetrainingexample,yu2025dapo,yuan2025vapo,zhang2025srpo}. 
A substantial body of recent work focuses on stabilizing and accelerating policy optimization under sparse verifiable rewards, including PPO-style objectives and credit-assignment refinements~\citep{schulman2017ppo,kazemnejad2024vineppo,yuan2025ppocollapse,yuan2025vapo,li2025turnppo}, as well as GRPO-style optimization and guided/regularized variants~\citep{liu2025understandingr1zero,zhang2025srpo,yu2025dapo,chen2025sgpo}. 
Overall, these methods primarily refine advantage/value estimation and optimization dynamics, providing a strong algorithmic foundation for exploiting verifiable reward signals.

In contrast, we focus on a complementary question that has received less systematic, mechanism-level study: \emph{where reliable learning signals come from in RLVR when the number of training prompts is extremely limited}. 
In the very-low-data regime, RL updates can be highly sensitive to sampling noise, making it crucial to identify prompts that yield consistent and directional gradient signals. 
We provide a mechanism-level account of prompt usefulness and show that deliberately pairing prompts can induce bidirectional supervision: one prompt produces rare successes that serve as positive anchors, while the other produces rare high-confidence failures that provide negative warnings.

\paragraph{Negative reinforcement and suppressing incorrect trajectories.}
A growing line of work argues that RLVR can benefit substantially from emphasizing negative learning signals, e.g., by penalizing incorrect trajectories or explicitly augmenting negative samples. Such approaches can suppress wrong generations and redistribute probability mass toward other plausible solutions under the model prior~\citep{zhu2025negativeRLVR,yang2025negativesampleaugmentation,chen2025sgpo,feng2025dontwastemistakes,arnal2025asymmetricreinforce}. 
Our method is related but distinct: rather than strengthening only the ``do-not'' signal, we construct a paired training unit that yields both ``do'' and ``do-not'' supervision within each update. Concretely, we combine an easy-but-brittle prompt (inducing rare failures) with a hard-but-solvable prompt (inducing rare successes), and implement the resulting bidirectional learning through WGRPO. In our experiments, this design improves directional consistency and training stability while preserving exploration.

\paragraph{Data Selection for LLM Post-Training.}
Data selection is a well-established topic in LLM post-training~\citep{ivison2025largescale}, with most work centered on supervised fine-tuning.
Common approaches include model-based filtering for quality~\citep{chen2024alpagasus}, selecting prompts using training-time signals~\citep{ivison2023dataefficient}, and gradient-related criteria for identifying influential data~\citep{xia2024less}.
Related lines in RLHF study how to select or refine preference data to reduce annotation cost while improving alignment outcomes~\citep{muldrew2024active,liu2024enabling,das2024active}.

Within RLVR, recent studies begin to ask which prompts are worth spending RL updates on.
LIMR shows that smaller curated RL prompt sets can match or outperform naive scaling, emphasizing data quality~\citep{li2025limr}.
One-shot RLVR further demonstrates that even a single carefully chosen prompt can improve reasoning and proposes selecting high-variance prompts~\citep{wang2025onetrainingexample}.
While effective, such heuristics can be setting-dependent and do not explicitly distinguish between prompts that provide stable positive anchors versus stable negative warnings.
Our work complements this line by proposing a mechanism-grounded minimal training design: positive--negative pairing instantiated with WGRPO that deliberately induces both positive and negative learning signals, making RLVR updates more interpretable and stable in the extreme low-data regime. We further show that this pairing criterion generalizes across models and benchmarks, indicating a transferable method for selecting informative prompts rather than an idiosyncratic heuristic.

%% file: sec/04_method.tex
\section{Method}
\paragraph{Overview.}
Our method has two components.
(i) \textbf{WGRPO} maps binary outcome feedback into \emph{weighted signed} outcomes and then applies \emph{group normalization} across $G$ rollouts for the same prompt, which automatically amplifies rare but informative events.
(ii) \textbf{Positive--Negative Pairing} chooses \emph{two} training prompts per update: a hard-but-solvable prompt whose \emph{rare successes} act as positive anchors, and an easy-but-brittle prompt whose \emph{rare failures} act as negative guidance.

\subsection{Mechanism: Rare-Event Amplification under Group-Normalized Weighted Outcomes}
\label{sec:mechanism}

\paragraph{Setup.}
Given a prompt $q$, GRPO samples a group of $G$ responses $\{o_i\}_{i=1}^G \sim \pi_{\theta_{\text{old}}}(\cdot\mid q)$ from the old policy.
Let $\{r_{i,t}\}_{t=1}^T$ be token-level rewards ($T$ is the maximum response length) and $\{m_{i,t}\}_{t=1}^T$ the end-of-sequence (EOS) mask over valid tokens.
We aggregate rewards into a scalar score
$
s_i^{\text{raw}}=\sum_{t=1}^{T} r_{i,t}\, m_{i,t}.
$
We focus on outcome-level RLVR: $r_{i,t}=0$ for non-terminal tokens and equals the outcome reward at EOS (equivalently, broadcast over valid tokens), so $s_i^{\text{raw}}$ reduces to an outcome score.
This notation also covers signed (e.g., $\{-1,+1\}$) and binary (e.g., $\{0,1\}$) rewards.

\paragraph{Weighted binary outcomes.}
We map each trajectory to a weighted binary outcome
\begin{equation}
y_i \;=\;
\begin{cases}
+1, & s_i^{\text{raw}}>\tau,\\
-\lambda_{\text{neg}}, & \text{otherwise},
\end{cases}
\label{eq:wgrpo_map_rewrite}
\end{equation}
where $\tau$ is a task-dependent threshold (e.g., $\tau=0$ for signed rewards, $\tau=0.5$ for $\{0,1\}$ rewards). We use $\tau=0$ in this paper.
$\lambda_{\text{neg}}>0$ controls the relative magnitude of the penalty assigned to incorrect trajectories prior to group normalization.
After normalization, the advantage geometry is primarily governed by the empirical group success rate, while $\lambda_{\text{neg}}$ mainly affects gradient scaling in finite-precision implementations and interacts with $\varepsilon_{\text{std}}$ when the group is near-degenerate or when the underlying outcome signal is not strictly binary. See Appendix~\ref{app:degenerate} for details on degenerate/near-degenerate groups and the effect of $\lambda_{\mathrm{neg}}$.

\paragraph{Group-normalized outcome advantages.}
Following GRPO, we compute the group-wise mean and standard deviation for prompt $q$:
\[
\mu_q \;=\; \frac{1}{G}\sum_{i=1}^G y_i,\qquad
\sigma_q \;=\; \sqrt{\frac{1}{G}\sum_{i=1}^G (y_i-\mu_q)^2}.
\]
We then broadcast the normalized outcome to tokens and apply the EOS mask:
\begin{equation}
A_{i,t}^{\text{WGRPO}} \;=\; \frac{y_i-\mu_q}{\sigma_q+\varepsilon_{\text{std}}}\; m_{i,t},
\qquad t=1,\dots,T,
\label{eq:wgrpo_adv_rewrite}
\end{equation}
where $\varepsilon_{\text{std}}>0$ is a small constant for numerical stability when $\sigma_q$ is close to $0$.

\paragraph{Optimization objective.}
We plug $A_{i,t}^{\text{WGRPO}}$ into the standard clipped GRPO objective and add a KL penalty:
\begin{multline}
\label{eq:wgrpo_clip_rewrite}
\mathcal{L}_{\mathrm{clip}}(\theta)
= -\frac{1}{G}\sum_{i=1}^{G}\frac{1}{T}\sum_{t=1}^{T}
\min\Big(
\rho_{i,t}(\theta)\,A^{\mathrm{WGRPO}}_{i,t},\\
\operatorname{clip}\big(\rho_{i,t}(\theta),\,1-\varepsilon_{\text{clip}},\,1+\varepsilon_{\text{clip}}\big)\,
A^{\mathrm{WGRPO}}_{i,t}
\Big),
\end{multline}
\begin{equation}
\mathcal{L}_{\mathrm{WGRPO}}(\theta)
= \mathcal{L}_{\mathrm{clip}}(\theta)
+ \beta\,\mathrm{KL}\big(\pi_\theta\,\|\,\pi_{\mathrm{ref}}\big),
\label{eq:wgrpo_obj_rewrite}
\end{equation}
where $\rho_{i,t}(\theta)=\frac{\pi_{\theta}(o_{i,t}\mid q, o_{i,<t})}{\pi_{\theta_{\mathrm{old}}}(o_{i,t}\mid q, o_{i,<t})}$ is the standard token-level likelihood ratio, $\varepsilon_{\text{clip}}$ is the clipping parameter, and $\beta>0$ is the KL coefficient.
The KL term is added outside the clipped minimum as a separate penalty. We adopt the common approximate KL formulation \citep{wang2025onetrainingexample} used widely in prior RLVR works \citep{shao2024deepseekmath,wang2025onetrainingexample,guo2025deepseekr1}.

\paragraph{Rare-event amplification induced by group normalization.}
Let $k$ be the number of correct responses in a group of size $G$ and $p=k/G$ be the empirical group success rate.
For the non-degenerate case $0<k<G$, the group statistics admit closed forms:
\begin{equation}
\mu_q = (1+\lambda_{\text{neg}})p - \lambda_{\text{neg}},
\label{eq:closed_form_mean_rewrite}
\end{equation}
\begin{equation}
\sigma_q = (1+\lambda_{\text{neg}})\sqrt{p(1-p)}.
\label{eq:closed_form_std_rewrite}
\end{equation}
The proof is provided in Appendix~\ref{app:closed_form}. Thus, the normalized advantages for correct and incorrect trajectories become
\begin{equation}
\label{eq:adv_with_eps_rewrite}
\begin{aligned}
A^{+}
&=
\frac{(1+\lambda_{\mathrm{neg}})(1-p)}
{(1+\lambda_{\mathrm{neg}})\sqrt{p(1-p)}+\varepsilon_{\text{std}}}, \\
A^{-}
&=
\frac{-(1+\lambda_{\mathrm{neg}})p}
{(1+\lambda_{\mathrm{neg}})\sqrt{p(1-p)}+\varepsilon_{\text{std}}}.
\end{aligned}
\end{equation}
When $\varepsilon_{\text{std}}$ is small relative to $(1+\lambda_{\text{neg}})\sqrt{p(1-p)}$, the geometry is dominated by $p$:
rare successes ($p\ll 1$) receive large positive advantages, while rare failures ($p\approx 1$) receive large negative advantages.
This yields an automatic \emph{rare-event amplification} effect driven by group normalization rather than an explicit curriculum.

\paragraph{Example with $G=8$.}
With $G=8$ and $\varepsilon_{\text{std}}\!\ll\!\sqrt{p(1-p)}$, the induced advantage geometry depends mainly on $p$:
hard groups with $p=1/8$ yield $A^+\approx 2.65$ and $A^-\approx -0.38$,
while easy groups with $p=7/8$ yield $A^+\approx 0.38$ and $A^-\approx -2.65$.
This illustrates the core mechanism: hard prompts amplify rare successes (positive anchors), whereas easy prompts amplify rare failures (negative guidance), producing an adaptive curriculum without explicit difficulty heuristics.

\subsection{Prompt Selection via Positive--Negative Pairing}
\label{sec:instance_selection}

\paragraph{Core idea.}
Instead of selecting training prompts solely by a single heuristic measure (e.g., historical-accuracy variance), we explicitly construct a bidirectional minibatch consisting of (i) one prompt that yields a stable positive anchor and (ii) one prompt that yields a stable negative warning.
Concretely, we select a two-prompt training set
$
\mathcal{D}_{\pm}=\{q^{+}, q^{-}\},
$
where $q^{+}$ is \emph{hard-but-solvable} (rare successes exist) and $q^{-}$ is \emph{easy-but-brittle} (rare failures exist).
Under WGRPO, these two regimes map directly to amplified tail-event teaching signals (Sec.~\ref{sec:mechanism}).

\paragraph{Positive anchor: hard-but-solvable.}
We choose $q^{+}$ such that the current policy achieves a low but non-zero success rate:
\begin{equation}
p(q^{+}) \in \Big[\tfrac{1}{G}, \tfrac{c}{G}\Big],
\label{eq:positive_anchor_p_rewrite}
\end{equation}
so that $0<k<G$ and each group typically contains a small number of correct rollouts. $c$ controls the width of the target success-rate regimes used to identify $q^+$ and $q^-$
In this regime, WGRPO assigns large positive advantages to rare correct trajectories, concentrating updates on demonstrations of what the model \emph{should} do.

\paragraph{Negative guidance: easy-but-brittle.}
We choose $q^{-}$ such that the current policy achieves a high but not perfect success rate:
\begin{equation}
p(q^{-}) \in \Big[1-\tfrac{c}{G},\, 1-\tfrac{1}{G}\Big],
\label{eq:negative_anchor_p_rewrite}
\end{equation}
so that failures are rare but still occur.
In this regime, WGRPO assigns large-magnitude negative advantages to rare failures, producing a sharp ``do-not'' signal that suppresses failure modes while preserving alternative plausible solutions under the model prior.

\paragraph{Practical selection via lightweight probing.}
To instantiate positive--negative pairing with only two training prompts, we perform a simple probing stage on two candidate pools with different expected difficulty under the same base model.
We use an ``easy'' candidate pool $\mathcal{C}^{-}$ and a ``hard'' candidate pool $\mathcal{C}^{+}$; in our experiments $\mathcal{C}^{-}$ is drawn from DeepScaleR-sub and $\mathcal{C}^{+}$ is drawn from AIME 2025, but the procedure is agnostic to the specific sources.
For each candidate $q \in \mathcal{C}^{+} \cup \mathcal{C}^{-}$, we estimate its success rate under the current policy by sampling $M$ independent groups of size $G$ and averaging:
\[
\bar p(q) \;=\; \frac{1}{{M}}\sum_{m=1}^{M} \hat p_m(q).
\]
To ensure non-degenerate within-group variance, we discard candidates with
$
\bar p(q)\notin[\delta,1-\delta],
$
where we use $\delta=1/G$ by default.
We then select one positive anchor and one negative guidance prompt by targeting the two WGRPO regimes:
\begin{equation}
\begin{aligned}
q^{+} &= \arg\min_{q \in \mathcal{C}^{+}}
\left|\bar p(q) - p_{\mathrm{hard}}\right|, \\
q^{-} &= \arg\min_{q \in \mathcal{C}^{-}}
\left|\bar p(q) - p_{\mathrm{easy}}\right|,
\end{aligned}
\label{eq:two_example_selection_rewrite}
\end{equation}
where $p_{\mathrm{hard}} \approx 1/G$ and $p_{\mathrm{easy}} \approx 1 - 1/G$.
This ensures that $q^{+}$ operates in a low-but-nonzero success regime that amplifies rare successes, while $q^{-}$ operates in a high-but-not-perfect success regime that amplifies rare failures.
Overall, the selection is deliberately simple, uses only on-policy probing (no historical training statistics), and directly instantiates the rare-event amplification mechanism of WGRPO with only two training prompts.

%% file: sec/05_experimental_setup.tex
\section{Experimental Setup}
\label{sec:experimental_setup}

\paragraph{Models.}
To study how different training-prompt selection strategies affect RLVR, we run our training pipeline on several representative open-weight LLMs from different models.
In particular, we train Qwen2.5-Math-7B and Qwen2.5-Math-7B-Instruct.

\paragraph{Training dataset.}
The training prompts we select come from \textsc{AIME 2025}~\citep{aops2025aime} and DeepScaleR-sub~\citep{wang2025onetrainingexample}.
DeepScaleR-sub is a randomly sampled subset of 1209 training prompts from the DeepScaleR-Preview-Dataset~\citep{luo2025deepscaler}.
All rewards are outcome-based verifiable rewards computed by exact-answer checking.

\paragraph{Prompt selection.}
Baseline: high-variance. We follow the primary setting in~\citet{wang2025onetrainingexample} and select the two prompts with the highest historical training variance from DeepScaleR-sub, denoted as $\pi_{1}$ and $\pi_{2}$.
We train on $\{\pi_{1}, \pi_{2}\}$ using GRPO. Our method: bidirectional prompt selection. We follow the prompt selection procedure described in Sec.~\ref{sec:instance_selection} and likewise select two training prompts, denoted as $\pi_{1209}$ and $p_{12}$.
Among them, $\pi_{1209}$ is drawn from DeepScaleR-sub and corresponds to an easy-but-not-perfect prompt, while $p_{12}$ is drawn from \textsc{AIME 2025} and corresponds to a hard-but-solvable prompt.
Selection is performed once before RLVR training and kept fixed throughout training.
We train on $\{\pi_{1209}, p_{12}\}$ using WGRPO.

\paragraph{Probing overhead.}
The probing step required by our selection (estimating success rates under the current policy) is executed once prior to RLVR training using a lightweight rollout budget; its overhead is small relative to the total RLVR training rollouts and does not change the training-time compute budgets used in our method comparisons.

\paragraph{Training setup}
All trainings follow the \texttt{verl} framework~\citep{sheng2024hybridflow}.
The main training hyperparameters are summarized in Table~\ref{tab:impl_hparams}.
For each update, we sample a fixed number of prompts $B=2$, and for each prompt we generate $G=8$ responses\footnotemark\footnotetext{Since \texttt{drop\_last=True} is used in the training dataloader of \texttt{verl}, the dataset must contain no fewer samples than the \texttt{batch\_size}. Moreover, positive and negative prompts are required to be equally represented to ensure balanced positive and negative training signals. In low-data RLVR settings, we replicate the selected prompts in a symmetric manner until the dataset size matches the \texttt{batch\_size}, and use the replicated set as a new dataset.}. For two-prompt training, the same two prompts are used at every update.
For GRPO+DSR-sub, we train the base model using 1209 prompts. We train Qwen2.5-Math-7B and Qwen2.5-Math-7B-Instruct for at most 500 steps. By default, we do not apply early stopping and train all methods to the same maximum-step budget for a given base model. More details are provided in Appendix~\ref{app:repro_earlystop}.

\begin{figure*}[t]
  \centering
  \includegraphics[width=\textwidth]{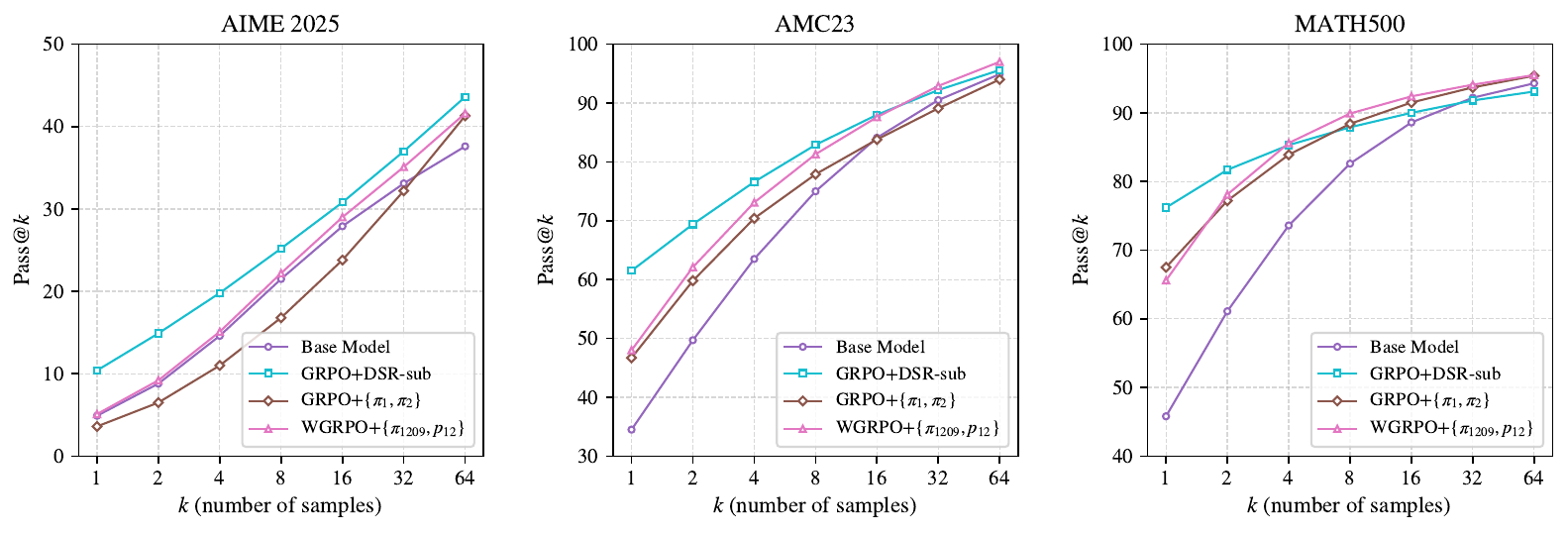}
  \caption{Pass@$k$ curves on \textsc{AIME 2025}, \textsc{AMC23}, and \textsc{MATH500} for Qwen2.5-Math-7B with Base Model, GRPO+DSR-sub, GRPO+$\{\pi_{1},\pi_{2}\}$, WGRPO+$\{\pi_{1209},p_{12}\}$.
  Except training on the large dataset (GRPO+DSR-sub), WGRPO+$\{\pi_{1209},p_{12}\}$ overall outperforms other methods across different $k$.
  On \textsc{AMC23} at $k=32,64$ and \textsc{MATH500} at $k=8,16,32,64$, WGRPO+$\{\pi_{1209},p_{12}\}$ shows the strongest performance.}
  \label{fig:passk_curves_qwen1}
\end{figure*}

\paragraph{Evaluation dataset and setup.}
We evaluate all trained models on \textsc{AIME 2025}, \textsc{AMC23}~\citep{aops_amc}, and \textsc{MATH500}~\citep{lightman2023verify}.
During prompt selection, one prompt $p_{12}$ from \textsc{AIME 2025} is used for training; we therefore remove $p_{12}$ from the evaluation set to avoid data leakage.
For evaluation, we set the maximum generation length to 3072 tokens.
We use the \texttt{qwen25-math-cot} prompt template for Qwen2.5-Math-7B and Qwen2.5-Math-7B-Instruct.
Unless otherwise specified, we set top\_p to 1, and use a temperature of 0.6; we use the same decoding setup for all methods.

To reduce evaluation variance, we replicate each evaluation problem at the dataset level (128× for AIME 2025/AMC23 and 8× for MATH500), effectively increasing the dataset size. This increases the effective number of samples per problem and stabilizes Pass@$k$ estimates, while keeping the underlying evaluation protocol unchanged.
We still generate $n=64$ responses per original problem and compute Pass@$k$ at the problem level using these $n$ samples.
We use Pass@$k$ as the evaluation metric, with $k \in \{1, 2, 4, 8, 16, 32, 64\}$.
To reduce variance in estimating Pass@$k$, we adopt the unbiased estimator from~\citet{chen2021evaluating}, which uses $n \ge k$ samples per problem and computes an unbiased estimate based on the number of correct responses:
\begin{equation}
\mathrm{Pass@}k
= \mathbb{E}_{x \sim \mathcal{C}}
\left[
1 - \frac{\binom{n-c}{k}}{\binom{n}{k}}
\right].
\label{eq:pass_at_k}
\end{equation}

\section{Results and Analysis}
\label{sec:results}

\begin{figure*}[t]
  \centering
  \includegraphics[width=\textwidth]{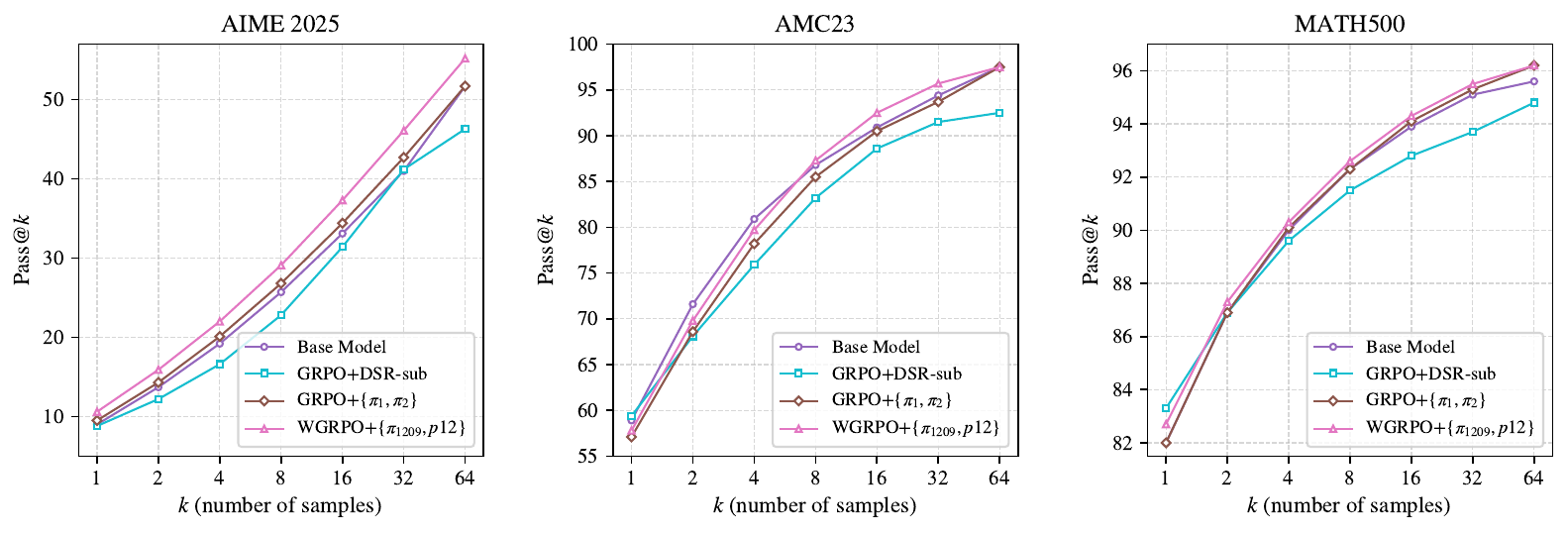}
  \caption{Pass@$k$ curves on \textsc{AIME 2025}, \textsc{AMC23}, and \textsc{MATH500} for Qwen2.5-Math-7B-Instruct with Base Model, GRPO+DSR-sub, GRPO+$\{\pi_{1},\pi_{2}\}$, WGRPO+$\{\pi_{1209},p_{12}\}$.
  WGRPO+$\{\pi_{1209},p_{12}\}$ is comparable to other methods, but shows distinct gains on \textsc{AIME 2025}.}
  \label{fig:passk_curves_qwen2}
\end{figure*}

\begin{table*}[t]
\centering
\caption{Pass@$k$ results on AIME 2025, AMC23, and MATH500 for Qwen2.5-Math-7B under different prompt selection strategies. ${\pi_1,\pi_2}$ are the two highest-variance prompts selected by the prior variance-based baseline~\cite{wang2025onetrainingexample}. ${\pi_{1209}, p_{12}}$ are selected by our bidirectional prompt selection method (Section~\ref{sec:instance_selection}). Bold and underlined numbers denote the best and second-best results for each $k$. Additional results for a broader comparison of prompt selection strategies are reported in Appendix~\ref{app:training_results}.}
\label{tab:passk_qwen_math_7b}

\small
\setlength{\tabcolsep}{6pt}
\renewcommand{\arraystretch}{1.08}

\begin{tabularx}{\textwidth}{l c *{7}{Y}}
\toprule
\textbf{Method} & \textbf{Size} & \multicolumn{7}{c}{\textbf{Pass@$k$}} \\
\cmidrule(lr){3-9}
$k$ &  & 1 & 2 & 4 & 8 & 16 & 32 & 64 \\
\midrule

&  & \multicolumn{7}{c}{\textbf{AIME 2025}} \\

Base Model & 0 & 4.9 & 8.8 & 14.6 & 21.5 & 27.9 & 33.1 & 37.6 \\
GRPO+DSR-sub & 1209 & \textbf{10.4} & \textbf{14.9} & \textbf{19.8} & \textbf{25.2} & \textbf{30.8} & \textbf{37.0} & \textbf{43.6} \\
GRPO+$\{\pi_{1209}, p_{12}\}$ & 2 & 2.9 & 5.5 & 9.8 & 16.3 & 24.1 & 32.2 & 40.9 \\
WGRPO+$\{\pi_{1}, p_{12}\}$ & 2 & \underline{5.1} & 9.0 & 14.0 & 19.8 & 26.1 & 33.0 & 40.4 \\
GRPO+$\{\pi_{1}, \pi_{2}\}$ & 2 & 3.6 & 6.5 & 11.0 & 16.8 & 23.8 & 32.2 & 41.3 \\
\rowcolor{rowgray}
WGRPO+$\{\pi_{1209}, p_{12}\}$ & 2 & \underline{5.1} & \underline{9.2} & \underline{15.1} & \underline{22.2} & \underline{29.0} & \underline{35.1} & \underline{41.6} \\
\midrule

&  & \multicolumn{7}{c}{\textbf{AMC23}} \\

Base Model & 0 & 34.5 & 49.7 & 63.5 & 75.0 & 84.1 & 90.5 & 94.9 \\
GRPO+DSR-sub & 1209 & \textbf{61.5} & \textbf{69.4} & \textbf{76.6} & \textbf{82.9} & \textbf{88.0} & 92.2 & 95.6 \\
GRPO+$\{\pi_{1209}, p_{12}\}$ & 2 & 47.6 & 61.2 & \underline{73.1} & 80.6 & 87.4 & \underline{92.5} & \underline{96.6} \\
WGRPO+$\{\pi_{1}, p_{12}\}$ & 2 & 44.7 & 60.7 & 72.4 & 80.5 & 86.9 & 92.3 & 96.4 \\
GRPO+$\{\pi_{1}, \pi_{2}\}$ & 2 & 46.7 & 59.8 & 70.4 & 77.9 & 83.8 & 89.1 & 94.0 \\
\rowcolor{rowgray}
WGRPO+$\{\pi_{1209}, p_{12}\}$ & 2 & \underline{48.0} & \underline{62.1} & \underline{73.1} & \underline{81.3} & \underline{87.6} & \textbf{92.9} & \textbf{97.0} \\
\midrule

&  & \multicolumn{7}{c}{\textbf{MATH500}} \\

Base Model & 0 & 45.8 & 61.1 & 73.6 & 82.6 & 88.6 & 92.2 & 94.3 \\
GRPO+DSR-sub & 1209 & \textbf{76.2} & \textbf{81.7} & \underline{85.3} & 87.9 & 90.0 & 91.8 & 93.1 \\
GRPO+$\{\pi_{1209}, p_{12}\}$ & 2 & 58.3 & 71.0 & 80.2 & 86.5 & 90.6 & 93.4 & 95.1 \\
WGRPO+$\{\pi_{1}, p_{12}\}$ & 2 & 61.5 & 75.7 & 84.5 & \underline{89.5} & \underline{92.3} & \textbf{94.2} & \textbf{95.5} \\
GRPO+$\{\pi_{1}, \pi_{2}\}$ & 2 & \underline{67.5} & 77.2 & 83.9 & 88.4 & 91.5 & 93.7 & \underline{95.4} \\
\rowcolor{rowgray}
WGRPO+$\{\pi_{1209}, p_{12}\}$ & 2 & 65.6 & \underline{78.1} & \textbf{85.6} & \textbf{89.9} & \textbf{92.4} & \underline{94.1} & \textbf{95.5} \\
\bottomrule
\end{tabularx}

\end{table*}

\paragraph{Compared methods.}
We report Pass@$k$ performance on \textsc{AIME 2025}, \textsc{AMC23}, and \textsc{MATH500}.
Base Model is the pretrained model without any RLVR training.
GRPO+DSR-sub trains the base model with GRPO using the 1209 prompts from DeepScaleR-sub pool. GRPO+$\{\pi_1,\pi_2\}$ is our main baseline: GRPO trained on two high-variance training prompts selected by the historical-accuracy-variance heuristic.
WGRPO+$\{\pi_{1209},p_{12}\}$ is our method: WGRPO combined with our easy+hard two-prompt selection, where $p_{12}$ is a hard training prompt and $\pi_{1209}$ is an easy training prompt.
We also include two ablations: GRPO+$\{\pi_{1209},p_{12}\}$ (replace WGRPO with GRPO) and WGRPO+$\{\pi_1,p_{12}\}$ (replace the easy prompt with a high variance prompt while keeping WGRPO).

\paragraph{Main comparison: our easy+hard based WGRPO consistently outperforms the high-variance-based GRPO baseline.}
As shown in Figure~\ref{fig:passk_curves_qwen1} and Figure~\ref{fig:passk_curves_qwen2}, WGRPO+$\{\pi_{1209},p_{12}\}$ consistently and meaningfully outperforms the baseline GRPO+$\{\pi_1,\pi_2\}$.
Across \textsc{AIME 2025} and \textsc{AMC23}, the gap is clear, especially at moderate $k$, with representative improvements such as 16.8$\rightarrow$22.2 on \textsc{AIME 2025} ($k{=}8$) and 94.0$\rightarrow$97.0 on \textsc{AMC23} ($k{=}64$).
The improvement on \textsc{MATH500} is smaller but remains generally consistent.
These results support a simple takeaway: high variance of historical accuracy is a weak proxy for informative training signal.
By design, the easy+hard pair provides complementary guidance under WGRPO: rare successes on the hard prompt produce a strong positive ``do'' signal, while rare failures on the easy prompt produce a strong negative ``do-not'' signal.
This bidirectional and low-ambiguity teaching signal yields a more stable optimization direction than variance-based selection.

\paragraph{Comparison to large-scale RLVR: competitive performance with $\mathbf{2}$ vs.\ $\mathbf{1209}$ training prompts.}
For Qwen2.5-Math-7B in Figure~\ref{fig:passk_curves_qwen1}, while GRPO+DSR-sub achieves the strongest results on \textsc{AIME 2025} across all $k$, our two-prompts method recovers a large portion of this gain (41.6 at $k{=}64$) using three orders of magnitude fewer training prompts.
More strikingly, on \textsc{AMC23}, WGRPO+$\{\pi_{1209},p_{12}\}$ surpasses GRPO+DSR-sub at larger $k$ (97.0 vs.\ 95.6 at $k{=}64$),
and on \textsc{MATH500} it becomes best or near-best for $k\ge 4$ (e.g., 89.9 vs.\ 87.9 at $k{=}8$; 92.4 vs.\ 90.0 at $k{=}16$).
For Qwen2.5-Math-7B-Instruct in Figure~\ref{fig:passk_curves_qwen2}, WGRPO+$\{\pi_{1209},p_{12}\}$ achieves the best results on \textsc{AIME 2025} across all $k$.
These results suggest that, beyond sheer scale, the quality and structure of the RLVR teaching signal can be a primary driver of generalizable reasoning improvements.

\paragraph{Ablation I (algorithm): WGRPO is necessary for effective two-prompt training.}
Replacing WGRPO with GRPO while keeping the same two training prompts (GRPO+$\{\pi_{1209},p_{12}\}$) leads to a consistent drop.
As shown in Table~\ref{tab:passk_qwen_math_7b},
on \textsc{AIME 2025}, WGRPO improves 16.3$\rightarrow$22.2 at $k{=}8$ and 24.1$\rightarrow$29.0 at $k{=}16$.
On \textsc{MATH500}, the gap is similarly notable (86.5$\rightarrow$89.9 at $k{=}8$; 90.6$\rightarrow$92.4 at $k{=}16$).
This indicates that, under extreme data scarcity, simply applying GRPO is insufficient: the training signal must be reshaped so that rare but meaningful outcomes dominate the update.

\paragraph{Ablation II (instance selection): easy+hard pairing is more reliable than mixing in high-variance prompts.}
Keeping WGRPO but replacing the easy prompt $\pi_{1209}$ with a high-variance prompt (WGRPO+$\{\pi_1,p_{12}\}$) generally hurts performance.
As shown in Table~\ref{tab:passk_qwen_math_7b},
on \textsc{AIME 2025}, the degradation is visible across $k$ (e.g., 22.2 vs.\ 19.8 at $k{=}8$; 41.6 vs.\ 40.4 at $k{=}64$).
On \textsc{AMC23}, our method also remains better at larger $k$ (97.0 vs.\ 96.4 at $k{=}64$).
On \textsc{MATH500}, the two variants are close at large $k$, but the easy+hard pairing is still competitive and typically better at moderate $k$ (e.g., 89.9 vs.\ 89.5 at $k{=}8$; 92.4 vs.\ 92.3 at $k{=}16$).
Overall, these results support that the ``easy negative guidance'' prompt should be stable rather than merely uncertain.

\paragraph{Mechanistic analysis: rare-event amplification turns one hard and one easy prompt into complementary teaching signals.}
WGRPO applies group-wise normalization to weighted outcomes, inducing an advantage magnitude that depends on the group success rate $p$.
Concretely, when a training prompt is hard ($p$ is small), correct trajectories become rare events and WGRPO assigns them a large positive advantage, producing a strong ``do this'' update.
Conversely, when a training prompt is easy ($p$ is large), incorrect trajectories become rare events and WGRPO assigns them a large negative advantage, producing a strong ``do-not'' update.
Our instance selection explicitly pairs one hard prompt ($p_{12}$) and one easy prompt ($\pi_{1209}$), so that each minibatch contains both a stable positive anchor and a stable negative warning.
This bidirectional teaching signal makes the gradient direction less ambiguous and reduces the chance that updates are dominated by sampling noise.

%% file: sec/06_conclusion.tex
\section{Conclusion}
We studied training-prompt selection for RLVR in an extreme low-data regime with sparse, binary verification, and found that common variance- or hardness-based heuristics can make updates overly sensitive to minibatch sampling noise. Motivated by a mechanism-level view of group-normalized outcomes in which learning is driven by rare tail events, we proposed positive--negative pairing: an easy-but-brittle prompt that induces rare failures paired with a hard-but-solvable prompt that induces rare successes, together with WGRPO to contrastively amplify these complementary outcomes into a bidirectional training signal under a fixed rollout budget. This design stabilizes update directions and improves the sample efficiency of low-data RLVR for mathematical reasoning. Empirically, on Qwen-family models we obtain consistent Pass@$k$ gains across \textsc{AIME 2025}, \textsc{AMC23}, and \textsc{MATH500} using only two fixed training prompts. Overall, our results suggest that when RLVR data are limited, carefully structuring the training signal is crucial.

%% file: sec/08_impact_statement.tex
\section*{Impact Statement}
This paper presents work whose goal is to advance the field of Machine
Learning. There are many potential societal consequences of our work, none
which we feel must be specifically highlighted here.

%% file: sec/07_appendix.tex
\appendix
\appendix
\onecolumn

\section{Derivation of closed-form rare-event amplification}
\label{app:closed_form}
\begin{proposition}
Let $k$ be the number of correct responses in a group of size $G$, and let
$p = k/G$ be the group success rate.
For $0 < k < G$, suppose the outcome mapping is
\[
y_i =
\begin{cases}
+1, & \text{if response } o_i \text{ is correct},\\
-\lambda_{\mathrm{neg}}, & \text{otherwise}.
\end{cases}
\]
Then the group-wise mean and standard deviation admit the following closed forms:
\[
\mu_q = (1+\lambda_{\mathrm{neg}})p - \lambda_{\mathrm{neg}}, \qquad
\sigma_q = (1+\lambda_{\mathrm{neg}})\sqrt{p(1-p)}.
\]
\end{proposition}

\begin{proof}
We derive the group-wise mean and standard deviation directly from their definitions.

\paragraph{Group-wise mean.}
By definition,
\[
\mu_q = \frac{1}{G}\sum_{i=1}^{G} y_i.
\]
Since there are $k$ correct responses and $G-k$ incorrect ones, we have
\[
\sum_{i=1}^{G} y_i
= k \cdot (+1) + (G-k)\cdot(-\lambda_{\mathrm{neg}}).
\]
Rearranging terms yields
\[
\sum_{i=1}^{G} y_i
= k(1+\lambda_{\mathrm{neg}}) - G\lambda_{\mathrm{neg}}.
\]
Dividing both sides by $G$ and substituting $p = k/G$, we obtain
\[
\mu_q = (1+\lambda_{\mathrm{neg}})p - \lambda_{\mathrm{neg}}.
\]

\paragraph{Group-wise standard deviation.}
The group-wise variance is defined as
\[
\sigma_q^2
= \frac{1}{G}\sum_{i=1}^{G} (y_i - \mu_q)^2.
\]
We consider the two types of responses separately.

For a correct response with $y_i = 1$, we have
\[
1 - \mu_q
= 1 - \big[(1+\lambda_{\mathrm{neg}})p - \lambda_{\mathrm{neg}}\big]
= (1+\lambda_{\mathrm{neg}})(1-p).
\]
For an incorrect response with $y_i = -\lambda_{\mathrm{neg}}$, we have
\[
-\lambda_{\mathrm{neg}} - \mu_q
= -\lambda_{\mathrm{neg}} - \big[(1+\lambda_{\mathrm{neg}})p - \lambda_{\mathrm{neg}}\big]
= -(1+\lambda_{\mathrm{neg}})p.
\]

Therefore,
\begin{align*}
\sigma_q^2
&= \frac{1}{G}\Big[
k(1+\lambda_{\mathrm{neg}})^2(1-p)^2
+ (G-k)(1+\lambda_{\mathrm{neg}})^2 p^2
\Big] \\
&= (1+\lambda_{\mathrm{neg}})^2
\Big[
p(1-p)^2 + (1-p)p^2
\Big].
\end{align*}
Factoring out $p(1-p)$ gives
\[
\sigma_q^2 = (1+\lambda_{\mathrm{neg}})^2 p(1-p).
\]
Taking the square root completes the proof:
\[
\sigma_q = (1+\lambda_{\mathrm{neg}})\sqrt{p(1-p)}.
\]
\end{proof}

\section{More details for Weighted GRPO}
\label{app:detail_WGRPO}

\subsection{Code for Weighted GRPO}
\label{app:code_WGRPO}
In this subsection, we provide the detailed implementation of WGRPO. Note that in the verl framework, the GRPO outcome-advantage routine is implemented as \texttt{compute\_grpo\_outcome\_advantage}, while our method corresponds to \texttt{compute\_wgrpo\_outcome\_advantage}. Since \texttt{verl} uses this function name in multiple places, renaming \texttt{compute\_grpo\_outcome\_advantage} requires modifying several related references. For simplicity, one can directly replace the implementation body with our \texttt{compute\_wgrpo\_outcome\_advantage} code while keeping the original function name \texttt{compute\_grpo\_outcome\_advantage}.
\medskip

\captionsetup{type=table}
\captionof{table}{Implementation of \texttt{compute\_wgrpo\_outcome\_advantage}.}
\label{tab:code_wgrpo_outcome_advantage}
\begin{icmlcodebox}
def compute_wgrpo_outcome_advantage(
    token_level_rewards: torch.Tensor,
    eos_mask: torch.Tensor,
    index: torch.Tensor,
    lambda_neg: float = 100.0,
    epsilon: float = 1e-6,
    reward_is_binary_01: bool = False,   # True: reward in {0,1}; False: reward in {+1,-1} or signed
):

    response_length = token_level_rewards.shape[-1]
    scores_raw = token_level_rewards.sum(dim=-1)

    if reward_is_binary_01:
        correct = scores_raw > 0.5
    else:
        correct = scores_raw > 0

    scores = torch.where(
        correct,
        torch.ones_like(scores_raw, dtype=torch.float32),
        -lambda_neg * torch.ones_like(scores_raw, dtype=torch.float32),
    )

    id2score = defaultdict(list)
    id2mean = {}
    id2std = {}

    with torch.no_grad():
        bsz = scores.shape[0]
        for i in range(bsz):
            id2score[index[i]].append(scores[i])

        for idx in id2score:
            if len(id2score[idx]) == 1:
                id2mean[idx] = torch.tensor(0.0, device=scores.device)
                id2std[idx] = torch.tensor(1.0, device=scores.device)
            else:
                stacked = torch.stack(id2score[idx]).to(scores.device)
                id2mean[idx] = stacked.mean()
                id2std[idx] = stacked.std(unbiased=False)

        for i in range(bsz):
            scores[i] = (scores[i] - id2mean[index[i]]) / (id2std[index[i]] + epsilon)
        advantages = scores.unsqueeze(-1).expand(-1, response_length) * eos_mask

    return advantages, advantages
\end{icmlcodebox}

\begin{table*}[t]
\centering
\caption{Pass@$k$ results on AIME 2025, AMC23, and MATH500 under different $\lambda_{\mathrm{neg}}$ settings. Bold and underlined numbers denote the best and second-best results for each $k$ (within each dataset).}
\label{tab:passk_lambda}

\small
\setlength{\tabcolsep}{6pt}
\renewcommand{\arraystretch}{1.08}

\begin{tabularx}{\textwidth}{l *{7}{Y}}
\toprule
\textbf{$\lambda_{\mathrm{neg}}$} & \multicolumn{7}{c}{\textbf{Pass@$k$}} \\
\cmidrule(lr){2-8}
$k$ & 1 & 2 & 4 & 8 & 16 & 32 & 64 \\
\midrule

& \multicolumn{7}{c}{\textbf{AIME 2025}} \\

$1$        & 2.9 & 5.5 & 9.8  & 16.3 & 24.1 & 32.2 & 40.9 \\
$50$       & \underline{4.6} & 8.3 & 13.9 & 20.9 & 28.4 & \underline{34.7} & \textbf{41.6} \\
$100$& \textbf{5.1} & \textbf{9.2} & \textbf{15.1} & \textbf{22.2} & \textbf{29.0} & \textbf{35.1} & \textbf{41.6} \\
$200$ & 4.5 & 8.1 & 13.5 & 20.3 & 27.6 & 34.4 & \textbf{41.6} \\
$500$      & \textbf{5.1} & \underline{9.1} & \underline{14.7} & \underline{21.8} & \underline{28.6} & 34.5 & \underline{41.3} \\
\midrule

& \multicolumn{7}{c}{\textbf{AMC23}} \\

$1$        & \underline{47.6} & 61.2 & \textbf{73.1} & 80.6 & 87.4 & 92.5 & 96.6 \\
$50$       & \textbf{48.0} & \underline{61.7} & \underline{72.8} & \underline{81.2} & \textbf{87.7} & \underline{92.7} & \underline{96.9} \\
$100$& \textbf{48.0} & \textbf{62.1} & \textbf{73.1} & \textbf{81.3} & \underline{87.6} & \textbf{92.9} & \textbf{97.0} \\
$200$ & 41.0 & 56.9 & 69.7 & 78.6 & 84.8 & 89.9 & 94.3 \\
$500$      & 47.5 & 56.5 & 71.3 & 80.6 & 85.9 & 90.6 & 95.1 \\
\midrule

& \multicolumn{7}{c}{\textbf{MATH500}} \\

$1$        & 58.3 & 71.0 & 80.2 & 86.5 & 90.6 & 93.4 & 95.1 \\
$50$       & 65.2 & \underline{77.6} & \underline{85.4} & 89.3 & 91.9 & 93.7 & 95.1 \\
$100$& \underline{65.6} & \textbf{78.1} & \textbf{85.6} & \textbf{89.9} & \textbf{92.4} & \underline{94.1} & \underline{95.5} \\
$200$ & 63.2 & 76.4 & 84.8 & \underline{89.6} & \underline{92.3} & \textbf{94.2} & \textbf{95.6} \\
$500$      & \textbf{65.8} & \textbf{78.1} & \underline{85.4} & 89.5 & 91.8 & 93.8 & 95.3 \\
\bottomrule
\end{tabularx}

\end{table*}

\subsection{Degenerate Groups and the Role of $\lambda_{\mathrm{neg}}$}
\label{app:degenerate}

This subsection clarifies the behavior of WGRPO in degenerate and near-degenerate groups, which are common under our low-data setting with group size $G$.
Recall Eq.~\ref{eq:wgrpo_adv_rewrite}:
\begin{equation*}
A^{\mathrm{WGRPO}}_{i,t}=\frac{y_i-\mu_q}{\sigma_q+\epsilon_{\mathrm{std}}}m_{i,t},
\end{equation*}
where $\epsilon_{\mathrm{std}}>0$ is a small constant for numerical stability when $\sigma_q$ is close to $0$.

\paragraph{Degenerate groups ($k\in\{0,G\}$).}
Let $k$ denote the number of correct responses within a group of size $G$.
When all $G$ responses are correct ($k=G$), we have $y_i\equiv +1$ for all $i$, hence $\mu_q=1$ and $\sigma_q=0$.
Therefore $y_i-\mu_q\equiv 0$ and $A^{\mathrm{WGRPO}}_{i,t}\equiv 0$ for all tokens.
Similarly, when all responses are incorrect ($k=0$), $y_i\equiv -\lambda_{\mathrm{neg}}$ for all $i$, so $\mu_q=-\lambda_{\mathrm{neg}}$ and $\sigma_q=0$, again yielding $A^{\mathrm{WGRPO}}_{i,t}\equiv 0$.
Thus, degenerate groups contribute no policy-gradient update (up to the KL term in our objective), which avoids spurious updates when within-group outcome variance collapses.

\paragraph{Why degenerate groups can be frequent but not harmful.}
In our two-prompt design we intentionally target regimes with $p\approx 1/G$ (hard-but-solvable) and $p\approx 1-1/G$ (easy-but-brittle).
With $G=8$, the probability of observing $k\in\{0,G\}$ is non-negligible, so degenerate groups can occur frequently.
However, these cases simply yield $A\equiv 0$ and effectively skip the policy-gradient update for that group.
Meanwhile, the same regimes place substantial probability mass on informative rare-event cases such as $k=1$ (rare success) or $k=G-1$ (rare failure), which provide strong ``do'' / ``do-not'' teaching signals.

\paragraph{Role of $\lambda_{\mathrm{neg}}$ in near-degenerate regimes.}
Our closed-form analysis in the main text focuses on non-degenerate groups ($0<k<G$).
In the idealized limit where $\epsilon_{\mathrm{std}}$ is negligible, the normalized advantage geometry is largely governed by $p$ and the factor $(1+\lambda_{\mathrm{neg}})$ cancels out (Eq.~\ref{eq:adv_with_eps_rewrite}).
In practice, groups with $k\in\{1,G-1\}$ can be near-degenerate and finite $\epsilon_{\mathrm{std}}$ (together with finite-precision arithmetic) prevents exact cancellation.
In this regime, $\lambda_{\mathrm{neg}}$ affects the scale of $y_i$ and thus the magnitude of $\sigma_q$, making normalization less sensitive to $\epsilon_{\mathrm{std}}$ and improving numerical robustness.
Unless specified otherwise, we use $\lambda_{\mathrm{neg}}=100$ and $\epsilon_{\mathrm{std}}=10^{-6}$ (Table~\ref{tab:impl_hparams}).

\subsection{Results under different $\lambda_{\mathrm{neg}}$ settings}
\label{app:lambda_neg_results}
We evaluate the sensitivity of WGRPO to the negative-weight coefficient $\lambda_{\mathrm{neg}}$, which defines the outcome mapping $y_i\in\{+1,-\lambda_{\mathrm{neg}}\}$ prior to within-group normalization (Appendix~\ref{app:code_WGRPO}). 
Across all three benchmarks (Table~\ref{tab:passk_lambda}), performance is consistently lower with $\lambda_{\mathrm{neg}}{=}1$, while choosing a larger $\lambda_{\mathrm{neg}}$ yields stable and generally improved Pass@$k$. A key observation is that, once $\lambda_{\mathrm{neg}}$ is moderately large (e.g., $\ge 50$ in our sweep), results vary only slightly across a broad range of values. This plateau suggests that WGRPO does not rely on fine-grained tuning of $\lambda_{\mathrm{neg}}$, instead, $\lambda_{\mathrm{neg}}$ mainly serves as a coarse scaling that separates negative outcomes strongly enough for robust normalization and learning. Practically, we therefore recommend using a reasonably large default (we use $\lambda_{\mathrm{neg}}{=}100$) rather than optimizing $\lambda_{\mathrm{neg}}$ per dataset.

\section{Additional Experimental Details and Reproducibility}
\label{app:repro}

\subsection{Training data and leakage control}
\label{app:repro_data}
Training candidates are drawn from \textsc{AIME 2025}~\citep{aops2025aime} and DeepScaleR-sub~\citep{wang2025onetrainingexample},
where DeepScaleR-sub contains 1209 prompts sampled from DeepScaleR-Preview-Dataset~\citep{luo2025deepscaler}.
In our easy+hard setting, one \textsc{AIME 2025} prompt is used for training; we therefore remove it from the \textsc{AIME 2025} evaluation set to avoid leakage.
All rewards are outcome-based verifiable rewards computed by exact-answer checking, producing binary rewards.

\subsection{Selection frequency and probing overhead}
\label{app:repro_probe}
Our easy+hard selection requires probing candidate prompts under the \emph{current policy} to estimate their empirical success rates.
In our experiments, this probing step is executed \emph{once before RLVR training} to select the final two prompts, and the selected pair is kept fixed throughout training.
The probing uses a lightweight rollout budget relative to RLVR training and therefore introduces only small additional overhead without changing the training-time rollout budgets used for method comparisons.

\subsection{Implementation hyperparameters}
\label{app:repro_hparams}
Table~\ref{tab:impl_hparams} lists the key hyperparameters needed to reproduce GRPO/WGRPO training. We report primal values that are fixed across experiments.

\begin{table}[t]
\centering
\small
\setlength{\tabcolsep}{8pt}
\renewcommand{\arraystretch}{1.1}
\caption{Implementation hyperparameters for GRPO/WGRPO training.}
\label{tab:impl_hparams}
\begin{tabular}{l p{6.2cm}}
\toprule
\textbf{Item} & \textbf{Value} \\
\midrule
Framework & \texttt{verl}~\citep{sheng2024hybridflow} \\
Max context length & 4096 \\
Learning rate & $1\times 10^{-6}$ \\
Group size ($G$ responses per prompt) & 8 \\
Max training steps & 500 \\
Hardware budget & $\le$ 8 H100 GPUs \\
Optimizer & AdamW (verl default) \\
Adam $\beta_1,\beta_2$ & 0.9, 0.95 \\
Weight decay & 0.01 \\
LR scheduler / warmup & none \\
Gradient clipping & 1 \\
Max generation tokens & 3072 \\
Reward format & binary verifiable reward (exact match) \\
KL coefficient & 0.001 \\
Negative-weight coefficient $\lambda_{\text{neg}}$ & 100 \\
Numerical stability constant $\varepsilon_{\text{std}}$ & $1 \times 10^{-6}$ \\
Independent groups of samping G & 8 \\
\bottomrule
\end{tabular}
\end{table}

\subsection{Checkpointing and early termination}
\label{app:repro_earlystop}
By default, we do not apply early stopping and train all methods to the same maximum-step budget for a given base model.
We only enable early termination to prevent irreversible collapse, and the rule is applied identically to all methods:
we stop training if the evaluation performance decreases monotonically for $K$ consecutive evaluation checkpoints (we use $K=5$).
When early termination is triggered, we use the best checkpoint along the training trajectory under the fixed evaluation protocol.

\subsection{Selected prompts for two-prompt training}
\label{app:repro_selected}
For transparency, we provide the full prompt text, answer, and source dataset for the selected prompts in Appendix~\ref{sec:selected_examples},
enabling exact reproduction of two-prompt training setting.

\begin{table*}[t]
\centering
\caption{Pass@$k$ results on AIME 2025, AMC23, and MATH500 for Qwen2.5-Math-7B under different prompt selection strategies. The three baseline pairs $\{\pi_{1},\pi_{2}\}$, $\{\pi_{1},\pi_{3}\}$, and $\{\pi_{2},\pi_{3}\}$ are selected by the prior variance-based baseline~\cite{wang2025onetrainingexample} and trained with GRPO. All other pairs are selected by our bidirectional prompt selection method (Section~\ref{sec:instance_selection}) and trained with WGRPO. Bold and underlined numbers denote the best and second-best results for each $k$ (within each dataset).}
\label{tab:passk_qwen_math_7b_pairs}

\small
\setlength{\tabcolsep}{6pt}
\renewcommand{\arraystretch}{1.08}

\begin{tabularx}{\textwidth}{l *{7}{Y}}
\toprule
\textbf{Prompt set} & \multicolumn{7}{c}{\textbf{Pass@$k$}} \\
\cmidrule(lr){2-8}
$k$ & 1 & 2 & 4 & 8 & 16 & 32 & 64 \\
\midrule

& \multicolumn{7}{c}{\textbf{AIME 2025}} \\

$\{\pi_{1}, \pi_{2}\}$ (baseline) & 3.6 & 6.5 & 11.0 & 16.8 & 23.8 & 32.2 & 41.3 \\
$\{\pi_{1}, \pi_{3}\}$ (baseline) & 4.1 & 7.2 & 12.5 & 17.1 & 23.7 & 32.1 & 40.6 \\
$\{\pi_{2}, \pi_{3}\}$ (baseline) & 3.5 & 6.3 & 10.5 & 16.2 & 22.9 & 31.0 & 40.1 \\
$\{\pi_{60}, p_{26}\}$            & \textbf{10.7} & \textbf{16.1} & \underline{21.6} & \underline{26.5} & \underline{31.2} & \underline{36.8} & \underline{43.3} \\
$\{\pi_{1033}, p_{29}\}$          & \underline{10.2} & \underline{15.8} & \textbf{21.8} & \textbf{27.7} & \textbf{33.6} & \textbf{39.9} & \textbf{46.2} \\
$\{\pi_{682}, p_{20}\}$           & 7.9 & 12.4 & 17.7 & 23.3 & 28.7 & 34.5 & 41.2 \\
$\{\pi_{1209}, p_{12}\}$          & 5.1 & 9.2 & 15.1 & 22.2 & 29.0 & 35.1 & 41.6 \\
$\{\pi_{1}, p_{12}\}$             & 5.1 & 9.0 & 14.0 & 19.8 & 26.1 & 33.0 & 40.4 \\
\midrule

& \multicolumn{7}{c}{\textbf{AMC23}} \\

$\{\pi_{1}, \pi_{2}\}$ (baseline) & 46.7 & 59.8 & 70.4 & 77.9 & 83.8 & 89.1 & 94.0 \\
$\{\pi_{1}, \pi_{3}\}$ (baseline) & 47.1 & 60.3 & 71.0 & 77.8 & 83.4 & 88.6 & 93.5 \\
$\{\pi_{2}, \pi_{3}\}$ (baseline) & 46.5 & 59.5 & 70.1 & 77.3 & 83.2 & 88.5 & 93.3 \\
$\{\pi_{60}, p_{26}\}$            & \textbf{61.2} & \textbf{70.9} & \textbf{78.1} & \textbf{84.3} & \textbf{89.3} & \textbf{93.3} & \underline{96.5} \\
$\{\pi_{1033}, p_{29}\}$          & \underline{60.5} & \underline{69.8} & \underline{77.0} & \underline{83.2} & \underline{88.5} & \underline{93.0} & 96.4 \\
$\{\pi_{682}, p_{20}\}$           & 57.4 & 68.1 & 75.7 & 81.7 & 86.2 & 90.1 & 94.2 \\
$\{\pi_{1209}, p_{12}\}$          & 48.0 & 62.1 & 73.1 & 81.3 & 87.6 & 92.9 & \textbf{97.0} \\
$\{\pi_{1}, p_{12}\}$             & 44.7 & 60.7 & 72.4 & 80.5 & 86.9 & 92.3 & 96.4 \\
\midrule

& \multicolumn{7}{c}{\textbf{MATH500}} \\

$\{\pi_{1}, \pi_{2}\}$ (baseline) & 67.5 & 77.2 & 83.9 & 88.4 & 91.5 & 93.7 & 95.4 \\
$\{\pi_{1}, \pi_{3}\}$ (baseline) & 68.0 & 77.5 & 83.9 & 88.2 & 91.3 & 93.4 & 95.1 \\
$\{\pi_{2}, \pi_{3}\}$ (baseline) & 67.3 & 77.0 & 83.6 & 87.8 & 91.1 & 93.2 & 95.2 \\
$\{\pi_{60}, p_{26}\}$            & \underline{75.1} & \textbf{82.7} & \textbf{87.5} & \textbf{90.5} & \textbf{92.7} & \textbf{94.3} & \textbf{95.7} \\
$\{\pi_{1033}, p_{29}\}$          & \textbf{75.5} & \textbf{82.7} & \underline{87.4} & \textbf{90.5} & \textbf{92.7} & \textbf{94.3} & \underline{95.5} \\
$\{\pi_{682}, p_{20}\}$           & 72.6 & \underline{81.5} & 86.9 & \underline{90.3} & \underline{92.4} & 93.9 & 95.1 \\
$\{\pi_{1209}, p_{12}\}$          & 65.6 & 78.1 & 85.6 & 89.9 & \underline{92.4} & 94.1 & \underline{95.5} \\
$\{\pi_{1}, p_{12}\}$             & 61.5 & 75.7 & 84.5 & 89.5 & 92.3 & \underline{94.2} & \underline{95.5} \\
\bottomrule
\end{tabularx}

\end{table*}

\section{Additional paired-prompt results}
\label{app:training_results}

Table~\ref{tab:passk_qwen_math_7b_pairs} reports results from multiple prompt pairs to assess whether the gains of our bidirectional prompt selection are robust to the specific choice of prompts.
For the variance-based baseline, we evaluate three representative pairs $\{\pi_1,\pi_2\}$, $\{\pi_1,\pi_3\}$, and $\{\pi_2,\pi_3\}$ to mitigate concerns that a single pair may reflect randomness.
For our method, we report five pairs selected by the proposed criterion (Section~\ref{sec:instance_selection}), aiming to test generalization across different easy--hard combinations rather than optimizing for one particular pair.
Across AIME~2025, AMC23, and MATH500, the pairs chosen by our method consistently achieve strong Pass@$k$ performance, and in most cases outperform the baseline pairs, especially at moderate to large $k$.
These results suggest that the benefit of bidirectional pairing is not tied to a single hand-picked prompt pair, but generalizes across multiple selected pairs.

\section{Limitations}

\paragraph{Model scale limitation.}
Our experiments are constrained by computational resources, so we focus on the 7B (e.g., Qwen2.5-Math-7B and Qwen2.5-Math-7B-Instruct) and do not run full RLVR training for larger models, like 14B/32B/72B models. As a result, future works should examine how the proposed easy–hard pairing and its group-normalization–induced bidirectional teaching signals scale to substantially larger models.

\paragraph{Task coverage limitation.}
We focus on mathematical reasoning and do not evaluate other RLVR tasks with deterministic outcomes, such as code generation. Still, our method shows strong generalization within math: although we select only two training prompts from AIME 2025 and DeepScaleR-sub, the trained model improves on AMC23 and MATH500 over the base model. A natural next step is to test whether the same prompt-selection strategy transfers to coding and other deterministic tasks.

\section{Details of selected prompts}

\label{sec:selected_examples}

{
\centering
\captionof{table}{Details of prompt $\pi_{1}$.}
\label{tab:pi1_example}
\small
\begin{tabularx}{\linewidth}{@{}X@{}}
\toprule
\textbf{Prompt:} \\
\midrule
{\ttfamily\footnotesize
The pressure \textbackslash\textbackslash( P \textbackslash\textbackslash) exerted by wind on a sail varies jointly as the area
\textbackslash\textbackslash( A \textbackslash\textbackslash) of the sail and the cube of the wind's velocity
\textbackslash\textbackslash( V \textbackslash\textbackslash).
When the velocity is \textbackslash\textbackslash( 8 \textbackslash\textbackslash) miles per hour, the pressure on a sail of
\textbackslash\textbackslash( 2 \textbackslash\textbackslash) square feet is \textbackslash\textbackslash( 4 \textbackslash\textbackslash) pounds.
Find the wind velocity when the pressure on \textbackslash\textbackslash( 4 \textbackslash\textbackslash) square feet of sail is
\textbackslash\textbackslash( 32 \textbackslash\textbackslash) pounds.
Let's think step by step and output the final answer within \textbackslash boxed\{\}.
} \\
\midrule
\textbf{Ground truth (label in DSR-sub):} 12.8. \\
\bottomrule
\end{tabularx}
\par
}

\addvspace{\floatsep}
\bigskip

{
\centering
\captionof{table}{Details of prompt $\pi_{2}$.}
\label{tab:pi2_details}
\small
\begin{tabularx}{\linewidth}{@{}X@{}}
\toprule
\textbf{Prompt:} \\
\midrule
{\ttfamily\footnotesize
How many positive divisors do 9240 and 13860 have in common?
Let's think step by step and output the final answer within \textbackslash boxed\{\}.
} \\
\midrule
\textbf{Ground truth (label in DSR-sub):} 24. \\
\bottomrule
\end{tabularx}
\par
}

\addvspace{\floatsep}
\bigskip

{
\centering
\captionof{table}{Details of prompt $\pi_{3}$.}
\label{tab:pi531_details}
\small
\begin{tabularx}{\linewidth}{@{}X@{}}
\toprule
\textbf{Prompt:} \\
\midrule
{\ttfamily\footnotesize
There are 10 people who want to choose a committee of 5 people among them. They do this by first electing a set of \$1,2,3\$, or 4 committee leaders, who then choose among the remaining people to complete the 5-person committee. In how many ways can the committee be formed, assuming that people are distinguishable? (Two committees that have the same members but different sets of leaders are considered to be distinct.) Let's think step by step and output the final answer within \textbackslash boxed\{\}.
} \\
\midrule
\textbf{Ground truth (label in DSR-sub):} 7560. \\
\bottomrule
\end{tabularx}
\par
}

\addvspace{\floatsep}
\bigskip

{
\centering
\captionof{table}{Details of prompt $\pi_{60}$.}
\label{tab:pi60_details}
\small

\begin{tabularx}{\linewidth}{@{}X@{}}
\toprule
\textbf{Prompt:} \\
\midrule
{\ttfamily\footnotesize
\detokenize{Given vectors $$\\overrightarrow {m}=( \\sqrt {3}\\sin x+\\cos x,1), \\overrightarrow {n}=(\\cos x,-f(x)), \\overrightarrow {m}\\perp \\overrightarrow {n}$$.\n(1) Find the monotonic intervals of $f(x)$;\n(2) Given that $A$ is an internal angle of $\\triangle ABC$, and $$f\\left( \\frac {A}{2}\\right)= \\frac {1}{2}+ \\frac { \\sqrt {3}}{2},a=1,b= \\sqrt {2}$$, find the area of $\\triangle ABC$. Let's think step by step and output the final answer within \\boxed{}.}
} \\
\midrule
\textbf{Ground truth (label in DSR-sub):} $\frac{\sqrt{3}-1}{4}$. \\
\bottomrule
\end{tabularx}
\par
}

\addvspace{\floatsep}
\bigskip

{
\centering
\captionof{table}{Details of prompt $\pi_{682}$.}
\label{tab:pi682_details}
\small

\begin{tabularx}{\linewidth}{@{}X@{}}
\toprule
\textbf{Prompt:} \\
\midrule
{\ttfamily\footnotesize
\detokenize{Given that circle $C$ passes through points $P(0,-4)$, $Q(2,0)$, and $R(3,-1)$.  \n$(1)$ Find the equation of circle $C$.  \n$(2)$ If the line $l: mx+y-1=0$ intersects circle $C$ at points $A$ and $B$, and $|AB|=4$, find the value of $m$. Let's think step by step and output the final answer within \\boxed{}.}
} \\
\midrule
\textbf{Ground truth (label in DSR-sub):} $\frac{4}{3}$. \\
\bottomrule
\end{tabularx}
\par
}

\addvspace{\floatsep}
\bigskip

{
\centering
\captionof{table}{Details of prompt $\pi_{1033}$.}
\label{tab:pi1033_details}
\small

\begin{tabularx}{\linewidth}{@{}X@{}}
\toprule
\textbf{Prompt:} \\
\midrule
{\ttfamily\footnotesize
\detokenize{In $\\triangle{ABC}$ with $AB = 12$, $BC = 13$, and $AC = 15$, let $M$ be a point on $\\overline{AC}$ such that the incircles of $\\triangle{ABM}$ and $\\triangle{BCM}$ have equal radii. Then $\\frac{AM}{CM} = \\frac{p}{q}$, where $p$ and $q$ are relatively prime positive integers. Find $p + q$. Let's think step by step and output the final answer within \\boxed{}}
} \\
\midrule
\textbf{Ground truth (label in DSR-sub):} 45. \\
\bottomrule
\end{tabularx}
\par
}

\addvspace{\floatsep}
\bigskip

{
\centering
\captionof{table}{Details of prompt $\pi_{1209}$.}
\label{tab:pi1209_details}
\small

\begin{tabularx}{\linewidth}{@{}X@{}}
\toprule
\textbf{Prompt:} \\
\midrule
{\ttfamily\footnotesize
\detokenize{Define the derivative of the $(n-1)$th derivative as the $n$th derivative $(n \\in N^{*}, n \\geqslant 2)$, that is, $f^{(n)}(x)=[f^{(n-1)}(x)]'$. They are denoted as $f''(x)$, $f'''(x)$, $f^{(4)}(x)$, ..., $f^{(n)}(x)$. If $f(x) = xe^{x}$, then the $2023$rd derivative of the function $f(x)$ at the point $(0, f^{(2023)}(0))$ has a $y$-intercept on the $x$-axis of ______. Let's think step by step and output the final answer within \\boxed{}}
} \\
\midrule
\textbf{Ground truth (label in DSR-sub):} $-\frac{2023}{2024}$. \\
\bottomrule
\end{tabularx}
\par
}

\addvspace{\floatsep}
\bigskip

{
\centering
\captionof{table}{Details of prompt $p_{12}$.}
\label{tab:p12_details}
\small

\begin{tabularx}{\linewidth}{@{}X@{}}
\toprule
\textbf{Prompt:} \\
\midrule
{\ttfamily\footnotesize
\detokenize{The set of points in 3-dimensional coordinate space that lie in the plane $x+y+z=75$ whose coordinates satisfy the inequalities $x-yz<y-zx<z-xy$ forms three disjoint convex regions. Exactly one of those regions has finite area. The area of this finite region can be expressed in the form $a\\sqrt{b}$, where $a$ and $b$ are positive integers and $b$ is not divisible by the square of any prime. Find $a+b$. Let's think step by step and output the final answer within \\boxed{}}
} \\
\midrule
\textbf{Ground truth (label in AIME 2025):} $510$. \\
\bottomrule
\end{tabularx}
\par
}

\addvspace{\floatsep}
\bigskip

{
\centering
\captionof{table}{Details of prompt $p_{20}$.}
\label{tab:p20_details}
\small

\begin{tabularx}{\linewidth}{@{}X@{}}
\toprule
\textbf{Prompt:} \\
\midrule
{\ttfamily\footnotesize
\detokenize{Circle $\\omega_1$ with radius 6 centered at point $A$ is internally tangent at point $B$ to circle $\\omega_2$ with radius 15. Points $C$ and $D$ lie on $\\omega_2$ such that $\\overline{BC}$ is a diameter of $\\omega_2$ and $\\overline{BC} \\perp \\overline{AD}$. The rectangle $EFGH$ is inscribed in $\\omega_1$ such that $\\overline{EF} \\perp \\overline{BC}$, $C$ is closer to $\\overline{GH}$ than to $\\overline{EF}$, and $D$ is closer to $\\overline{FG}$ than to $\\overline{EH}$, as shown. Triangles $\\triangle DGF$ and $\\triangle CHG$ have equal areas. The area of rectangle $EFGH$ is $\\frac{m}{n}$, where $m$ and $n$ are relatively prime positive integers. Find $m + n$. Let's think step by step and output the final answer within \\boxed{}.}
} \\
\midrule
\textbf{Ground truth (label in AIME 2025):} $293$. \\
\bottomrule
\end{tabularx}
\par
}

\addvspace{\floatsep}
\bigskip

{
\centering
\captionof{table}{Details of prompt $p_{26}$.}
\label{tab:p26_details}
\small

\begin{tabularx}{\linewidth}{@{}X@{}}
\toprule
\textbf{Prompt:} \\
\midrule
{\ttfamily\footnotesize
\detokenize{Let $ A_1A_2 \\ldots A_{11} $ be an 11-sided non-convex simple polygon with the following properties:\n* The area of $ A_iA_1A_{i+1} $ is 1 for each $ 2 \\leq i \\leq 10 $,\n* $ \\cos(\\angle A_iA_1A_{i+1}) = \\frac{12}{13} $ for each $ 2 \\leq i \\leq 10 $,\n* The perimeter of $ A_1A_2 \\ldots A_{11} $ is 20.\nIf $ A_1A_2 + A_1A_{11} $ can be expressed as $ \\frac{m\\sqrt{n} - p}{q} $ for positive integers $ m, n, p, q $ with $ n $ squarefree and no prime divides all of $ m, p, q$, find $ m + n + p + q $. Let's think step by step and output the final answer within \\boxed{}.}
} \\
\midrule
\textbf{Ground truth (label in AIME 2025):} $19$. \\
\bottomrule
\end{tabularx}
\par
}

\addvspace{\floatsep}
\bigskip

{
\centering
\captionof{table}{Details of prompt $p_{29}$.}
\label{tab:p29_details}
\small

\begin{tabularx}{\linewidth}{@{}X@{}}
\toprule
\textbf{Prompt:} \\
\midrule
{\ttfamily\footnotesize
\detokenize{There are exactly three positive real numbers $ k $ such that the function\n$ f(x) = \\frac{(x - 18)(x - 72)(x - 98)(x - k)}{x} $\ndefined over the positive real numbers achieves its minimum value at exactly two positive real numbers $ x $. Find the sum of these three values of $ k $. Let's think step by step and output the final answer within \\boxed{}.}
} \\
\midrule
\textbf{Ground truth (label in AIME 2025):} $240$. \\
\bottomrule
\end{tabularx}
\par
}

\addvspace{\floatsep}
\bigskip

%% file: icml2026.bib
@misc{guo2025deepseekr1,
  author = {Daya Guo and Dejian Yang and Haowei Zhang and Junxiao Song and
            Ruoyu Zhang and Runxin Xu and Qihao Zhu and Shirong Ma and
            Peiyi Wang and Xiao Bi and others},
  title = {DeepSeek-R1: Incentivizing Reasoning Capability in LLMs via Reinforcement Learning},
  year = {2025},
  howpublished = {arXiv preprint arXiv:2501.12948}
}

@misc{team2025kimik15,
  author = {{Kimi Team} and Angang Du and Bofei Gao and Bowei Xing and
            Changjiu Jiang and Cheng Chen and Cheng Li and Chenjun Xiao and
            Chenzhuang Du and Chonghua Liao and others},
  title = {Kimi k1.5: Scaling Reinforcement Learning with LLMs},
  year = {2025},
  howpublished = {arXiv preprint arXiv:2501.12599}
}

@misc{jaech2024o1,
  author = {Aaron Jaech and Adam Kalai and Adam Lerer and Adam Richardson and
            Ahmed El-Kishky and Aiden Low and Alec Helyar and Aleksander Madry and
            Alex Beutel and Alex Carney and others},
  title = {OpenAI o1 System Card},
  year = {2024},
  howpublished = {arXiv preprint arXiv:2412.16720}
}

@misc{lambert2024tulu3,
  author = {Nathan Lambert and Jacob Morrison and Valentina Pyatkin and
            Shengyi Huang and Hamish Ivison and Faeze Brahman and
            Lester James V Miranda and Alisa Liu and Nouha Dziri and
            Shane Lyu and others},
  title = {{T{\"U}LU}~3: Pushing Frontiers in Open Language Model Post-Training},
  year = {2024},
  howpublished = {arXiv preprint arXiv:2411.15124}
}

@misc{gao2024rewardreasoning,
  author = {Jiaxuan Gao and Shusheng Xu and Wenjie Ye and Weilin Liu and
            Chuyi He and Wei Fu and Zhiyu Mei and Guangju Wang and Yi Wu},
  title = {On Designing Effective RL Reward at Training Time for LLM Reasoning},
  year = {2024},
  howpublished = {arXiv preprint arXiv:2410.15115}
}

@inproceedings{miao2024inform,
  author    = {Yuchun Miao and Sen Zhang and Liang Ding and Rong Bao and
               Lefei Zhang and Dacheng Tao},
  title     = {{InfoRM}: Mitigating Reward Hacking in {RLHF} via Information-Theoretic Reward Modeling},
  booktitle = {Advances in Neural Information Processing Systems},
  year      = {2024}
}

@misc{cai2025verifiablenoisy,
  author = {Xin-Qiang Cai and Wei Wang and Feng Liu and Tongliang Liu and
            Gang Niu and Masashi Sugiyama},
  title = {Reinforcement Learning with Verifiable yet Noisy Rewards under Imperfect Verifiers},
  year = {2025},
  howpublished = {arXiv preprint arXiv:2510.00915}
}

@inproceedings{li2025processqvalue,
  author    = {Wendi Li and Yixuan Li},
  title     = {Process Reward Model with Q-Value Rankings},
  booktitle = {Proceedings of the Thirteenth International Conference on Learning Representations},
  year      = {2025}
}

@misc{zhang2025processlessons,
  author = {Zhenru Zhang and Chujie Zheng and Yangzhen Wu and Beichen Zhang and
            Runji Lin and Bowen Yu and Dayiheng Liu and Jingren Zhou and Junyang Lin},
  title = {The Lessons of Developing Process Reward Models in Mathematical Reasoning},
  year = {2025},
  howpublished = {arXiv preprint arXiv:2501.07301}
}

@misc{li2024numinamath,
  author = {Jia Li and Edward Beeching and Lewis Tunstall and Ben Lipkin and
            Roman Soletskyi and Shengyi Huang and Kashif Rasul and Longhui Yu and
            Albert Jiang and Ziju Shen and Zihan Qin and Bin Dong and Li Zhou and
            Yann Fleureau and Guillaume Lample and Stanislas Polu},
  title = {NuminaMath},
  year = {2024},
  howpublished = {Technical report and dataset, available at
                  https://huggingface.co/AI-MO/NuminaMath-CoT}
}

@misc{yu2025dapo,
  author = {Qiying Yu and Zheng Zhang and Ruofei Zhu and Yufeng Yuan and
            Xiaochen Zuo and Yu Yue and Tiantian Fan and Gaohong Liu and
            Lingjun Liu and Xin Liu and others},
  title = {DAPO: An Open-Source LLM Reinforcement Learning System at Scale},
  year = {2025},
  howpublished = {arXiv preprint arXiv:2503.14476}
}

@misc{li2025limr,
  author = {Xuefeng Li and Haoyang Zou and Pengfei Liu},
  title = {LIMR: Less Is More for RL Scaling},
  year = {2025},
  howpublished = {arXiv preprint arXiv:2502.11886}
}

@misc{wang2025onetrainingexample,
  author = {Yiping Wang and Qing Yang and Zhiyuan Zeng and Liliang Ren and
            Liyuan Liu and Baolin Peng and Hao Cheng and Xuehai He and
            Kuan Wang and Jianfeng Gao and Weizhu Chen and Shuohang Wang and
            Simon Shaolei Du and Yelong Shen},
  title = {Reinforcement Learning for Reasoning in Large Language Models with One Training Example},
  year = {2025},
  howpublished = {arXiv preprint arXiv:2504.20571}
}

@misc{zhu2025negativeRLVR,
  author = {Xinyu Zhu and Mengzhou Xia and Zhepei Wei and Wei-Lin Chen and Danqi Chen and Yu Meng},
  title = {The Surprising Effectiveness of Negative Reinforcement in LLM Reasoning},
  year = {2025},
  howpublished = {arXiv preprint arXiv:2506.01347}
}

@misc{yang2025negativesampleaugmentation,
  author = {Zhaohui Yang and Yuxiao Ye and Shilei Jiang and Chen Hu and Linjing Li and Shihong Deng and Daxin Jiang},
  title = {Unearthing Gems from Stones: Policy Optimization with Negative Sample Augmentation for {LLM} Reasoning},
  year = {2025},
  howpublished = {arXiv preprint arXiv:2505.14403}
}

@misc{chen2025sgpo,
  author = {Peter Chen and Xiaopeng Li and Ziniu Li and Xi Chen and Tianyi Lin},
  title = {Stepwise Guided Policy Optimization: Coloring your Incorrect Reasoning in {GRPO}},
  year = {2025},
  howpublished = {arXiv preprint arXiv:2505.11595}
}

@misc{feng2025dontwastemistakes,
  author = {Yunzhen Feng and Parag Jain and Anthony Hartshorn and Yaqi Duan and Julia Kempe},
  title = {Don't Waste Mistakes: Leveraging Negative {RL}-Groups via Confidence Reweighting},
  year = {2025},
  howpublished = {arXiv preprint arXiv:2510.08696}
}

@misc{arnal2025asymmetricreinforce,
  author = {Charles Arnal and Ga{\"e}tan Narozniak and Vivien Cabannes and Yunhao Tang and Julia Kempe and R{\'e}mi Munos},
  title = {Asymmetric {REINFORCE} for off-Policy Reinforcement Learning: Balancing Positive and Negative Rewards},
  year = {2025},
  howpublished = {arXiv preprint arXiv:2506.20520}
}

@misc{schulman2017ppo,
  author = {John Schulman and Filip Wolski and Prafulla Dhariwal and
            Alec Radford and Oleg Klimov},
  title = {Proximal Policy Optimization Algorithms},
  year = {2017},
  howpublished = {arXiv preprint arXiv:1707.06347}
}

@misc{kazemnejad2024vineppo,
  author = {Amirhossein Kazemnejad and Milad Aghajohari and Eva Portelance and
            Alessandro Sordoni and Siva Reddy and Aaron Courville and Nicolas Le Roux},
  title = {VinePPO: Unlocking RL Potential for LLM Reasoning through Refined Credit Assignment},
  year = {2024},
  howpublished = {arXiv preprint arXiv:2410.01679}
}

@misc{yuan2025ppocollapse,
  author = {Yufeng Yuan and Yu Yue and Ruofei Zhu and Tiantian Fan and Lin Yan},
  title = {What’s Behind PPO’s Collapse in Long-CoT? Value Optimization Holds the Secret},
  year = {2025},
  howpublished = {arXiv preprint arXiv:2503.01491}
}

@misc{yuan2025vapo,
  author = {Yufeng Yuan and Qiying Yu and Xiaochen Zuo and Ruofei Zhu and
            Wenyuan Xu and Jiaze Chen and Chengyi Wang and Tiantian Fan and
            Zhengyin Du and Xiangpeng Wei and others},
  title = {VAPO: Efficient and Reliable Reinforcement Learning for Advanced Reasoning Tasks},
  year = {2025},
  howpublished = {arXiv preprint arXiv:2504.05118}
}

@misc{li2025turnppo,
  author = {Junbo Li and Peng Zhou and Rui Meng and Meet P. Vadera and
            Lihong Li and Yang Li},
  title = {Turn-PPO: Turn-Level Advantage Estimation with PPO for Improved Multi-Turn RL in Agentic LLMs},
  year = {2025},
  howpublished = {arXiv preprint arXiv:2512.17008}
}

@misc{liu2025understandingr1zero,
  author = {Zichen Liu and Changyu Chen and Wenjun Li and Penghui Qi and
            Tianyu Pang and Chao Du and Wee Sun Lee and Min Lin},
  title = {Understanding R1-Zero-like Training: A Critical Perspective},
  year = {2025},
  howpublished = {arXiv preprint arXiv:2503.20783}
}

@misc{zhang2025srpo,
  author = {Xiaojing Zhang and Jinghui Wang and Zifei Cheng and Wenhao Zhuang and
            Zheng Lin and Minglei Zhang and Shaojie Wang and Yinghan Cui and
            Chao Wang and Junyi Peng and Shimiao Jiang and Shiqi Kuang and
            Shouyu Yin and Chaohang Wen and Haotian Zhang and Bin Chen and Bing Yu},
  title = {SRPO: A Cross-Domain Implementation of Large-Scale Reinforcement Learning on LLM},
  year = {2025},
  howpublished = {arXiv preprint arXiv:2504.14286}
}

@misc{shao2024deepseekmath,
  author       = {Zhihong Shao and Peiyi Wang and Qihao Zhu and Runxin Xu and Junxiao Song and Xiao Bi and Haowei Zhang and Mingchuan Zhang and Y. K. Li and Y. Wu and Daya Guo},
  title        = {DeepSeekMath: Pushing the Limits of Mathematical Reasoning in Open Language Models},
  year         = {2024},
  howpublished = {arXiv preprint arXiv:2402.03300}
}

@misc{primeintellect2025intellect2,
  author       = {{Prime Intellect Team} and Sami Jaghouar and Justus Mattern and Jack Min Ong and Jannik Straube and Manveer Basra and Aaron Pazdera and Kushal Thaman and Matthew Di Ferrante and Felix Gabriel and Fares Obeid and Kemal Erdem and Michael Keiblinger and Johannes Hagemann},
  title        = {INTELLECT-2: A Reasoning Model Trained Through Globally Decentralized Reinforcement Learning},
  year         = {2025},
  howpublished = {arXiv preprint arXiv:2505.07291}
}

@misc{xu2025elasticreasoning,
  author       = {Yuhui Xu and Hanze Dong and Lei Wang and Doyen Sahoo and Junnan Li and Caiming Xiong},
  title        = {Scalable Chain of Thoughts via Elastic Reasoning},
  year         = {2025},
  howpublished = {arXiv preprint arXiv:2505.05315}
}

@misc{wei2025truthrl,
  author       = {Zhepei Wei and Xiao Yang and Kai Sun and Jiaqi Wang and Rulin Shao and Sean Chen and Mohammad Kachuee and Teja Gollapudi and Tony Liao and Nicolas Scheffer and Rakesh Wanga and Anuj Kumar and Yu Meng and Wen-tau Yih and Xin Luna Dong},
  title        = {TruthRL: Incentivizing Truthful LLMs via Reinforcement Learning},
  year         = {2025},
  howpublished = {arXiv preprint arXiv:2509.25760}
}

@article{ivison2025largescale,
  title   = {Large-Scale Data Selection for Instruction Tuning},
  author  = {Ivison, Hamish and Zhang, Muru and Brahman, Faeze and Koh, Pang Wei and Dasigi, Pradeep},
  journal = {arXiv preprint arXiv:2503.01807},
  year    = {2025}
}

@inproceedings{chen2024alpagasus,
  title     = {AlpacaGAsus: Training a Better Alpaca with Fewer Data},
  author    = {Chen, Lichang and Li, Shiyang and Yan, Jun and Wang, Hai and Gunaratna, Kalpa and Yadav, Vikas and Tang, Zheng and Srinivasan, Vijay and Zhou, Tianyi and Huang, Heng and Jin, Hongxia},
  booktitle = {International Conference on Learning Representations},
  year      = {2024}
}

@inproceedings{ivison2023dataefficient,
  title     = {Data-Efficient Finetuning Using Cross-Task Nearest Neighbors},
  author    = {Ivison, Hamish and Smith, Noah A. and Hajishirzi, Hannaneh and Dasigi, Pradeep},
  booktitle = {Findings of the Association for Computational Linguistics},
  year      = {2023}
}

@inproceedings{xia2024less,
  title     = {LESS: Selecting Influential Data for Targeted Instruction Tuning},
  author    = {Xia, Mengzhou and Malladi, Sadhika and Gururangan, Suchin and Arora, Sanjeev and Chen, Danqi},
  booktitle = {International Conference on Machine Learning},
  year      = {2024}
}

@article{liu2024enabling,
  title   = {Enabling Weak LLMs to Judge Response Reliability via Meta Ranking},
  author  = {Liu, Zijun and Kou, Boqun and Li, Peng and Yan, Ming and Zhang, Ji and Huang, Fei and Liu, Yang},
  journal = {arXiv preprint arXiv:2402.12146},
  year    = {2024}
}

@article{das2024active,
  title   = {Active Preference Optimization for Sample-Efficient RLHF},
  author  = {Das, Nirjhar and Chakraborty, Souradip and Pacchiano, Aldo and Chowdhury, Sayak Ray},
  journal = {arXiv preprint arXiv:2402.10500},
  year    = {2025}
}

@inproceedings{muldrew2024active,
  title     = {Active Preference Learning for Large Language Models},
  author    = {Muldrew, William and Hayes, Peter and Zhang, Mingtian and Barber, David},
  booktitle = {International Conference on Machine Learning},
  year      = {2024}
}

@misc{luo2025deepscaler,
  author       = {Luo, Michael and Tan, Sijun and Wong, Justin and Shi, Xiaoxiang and Tang, William Y. and Roongta, Manan and Cai, Colin and Luo, Jeffrey and Li, Li Erran and Popa, Raluca Ada and Stoica, Ion},
  title        = {DeepScaler: Surpassing o1-preview with a 1.5B Model by Scaling RL},
  year         = {2025},
  howpublished = {Notion Blog},
  url          = {https://pretty-radio-b75.notion.site/DeepScaleR-Surpassing-O1-Preview-with-a-1-5B-Model-by-Scaling-RL-19681902c1468005bed8ca303013a4e2}
}

@misc{aops2025aime,
  author       = {{Art of Problem Solving}},
  title        = {AIME Problems and Solutions},
  year         = {2025},
  howpublished = {\url{https://artofproblemsolving.com/wiki/index.php/AIME_Problems_and_Solutions}},
  note         = {Accessed: 2025-04-20}
}

@misc{aops_amc,
  author       = {{Art of Problem Solving}},
  title        = {{AMC Problems and Solutions}},
  howpublished = {\url{https://artofproblemsolving.com/wiki/index.php?title=AMC_Problems_and_Solutions}},
  year         = {2025},
  note         = {Accessed: 2025-04-20}
}

@misc{lightman2023verify,
  author       = {Lightman, Hunter and Kosaraju, Vineet and Burda, Yura and Edwards, Harri and Baker, Bowen and Lee, Teddy and Leike, Jan and Schulman, John and Sutskever, Ilya and Cobbe, Karl},
  title        = {Let’s Verify Step by Step},
  year         = {2023},
  howpublished = {arXiv preprint arXiv:2305.20050}
}

@misc{sheng2024hybridflow,
  author       = {Sheng, Guangming and Zhang, Chi and Ye, Zilingfeng and Wu, Xibin and Zhang, Wang and Zhang, Ru and Peng, Yanghua and Lin, Haibin and Wu, Chuan},
  title        = {HybridFlow: A Flexible and Efficient {RLHF} Framework},
  year         = {2024},
  howpublished = {arXiv preprint arXiv:2409.19256}
}

@misc{chen2021evaluating,
  author       = {Chen, Mark and Tworek, Jerry and Jun, Heewoo and Yuan, Qiming and Pinto, Henrique Pond{\'e} de Oliveira and Kaplan, Jared and Edwards, Harri and Burda, Yuri and Joseph, Nicholas and Brockman, Greg and Ray, Alex and Puri, Raul and Krueger, Gretchen and Petrov, Michael and Khlaaf, Heidy and Sastry, Girish and Mishkin, Pamela and Chan, Brooke and Gray, Scott and Ryder, Nick and Pavlov, Mikhail and Power, Alethea and Kaiser, Lukasz and Bavarian, Mohammad and Winter, Clemens and Tillet, Philippe and Petroski Such, Felipe and Cummings, Dave and Plappert, Matthias and Chantzís, Fotios and Barnes, Elizabeth and Herbert-Voss, Ariel and Guss, William Hebgen and Nichol, Alex and Paino, Alex and Tezak, Nikolas and Tang, Jie and Babuschkin, Igor and Balaji, Suchir and Jain, Shantanu and Saunders, William and Hesse, Christopher and Carr, Andrew N. and Leike, Jan and Achiam, Joshua and Misra, Vedant and Morikawa, Evan and Radford, Alec and Knight, Matthew and Brundage, Miles and Murati, Mira and Mayer, Katie and Welinder, Peter and McGrew, Bob and Amodei, Dario and McCandlish, Sam and Sutskever, Ilya and Zaremba, Wojciech},
  title        = {Evaluating Large Language Models Trained on Code},
  year         = {2021},
  howpublished = {arXiv preprint arXiv:2107.03374}
}
